\documentclass{article}

\usepackage{microtype}
\usepackage{graphicx}
\usepackage{subfigure}
\usepackage{booktabs} 

\usepackage{hyperref}

\usepackage[accepted]{icml2023}

\usepackage{amsmath}
\usepackage{amssymb}
\usepackage{mathtools}
\usepackage{amsthm}

\usepackage[capitalize,noabbrev]{cleveref}

\newtheorem{theorem}{Theorem}[section]

\newtheorem{prop}[theorem]{Proposition}

\theoremstyle{definition}

\theoremstyle{remark}

\numberwithin{equation}{section}

\newcommand\R{\mathbb{R}}
\newcommand\Z{\mathbb{Z}}

\def\R{\mathbb R}

\usepackage[textsize=tiny]{todonotes}

\icmltitlerunning{Submission and Formatting Instructions for ICML 2023}

\begin{document}

\twocolumn[
\icmltitle{On the effectiveness of neural priors in modeling dynamical systems}



\icmlsetsymbol{equal}{*}

\begin{icmlauthorlist}
\icmlauthor{Sameera Ramasinghe}{equal,comp}
\icmlauthor{Hemanth Saratchandran}{equal,inst}
\icmlauthor{Violetta Shevchenko}{comp}
\icmlauthor{Simon Lucey}{inst}
\end{icmlauthorlist}


\icmlaffiliation{inst}{Australian Institute of Machine Learning, University of Adelaide, Adelaide SA, Australia}
\icmlaffiliation{comp}{Amazon, Australia}

\icmlcorrespondingauthor{Sameera Ramasinghe}{sameera.ramasinghe@adelaide.edu.au}
\icmlcorrespondingauthor{Hemanth Saratchandran}{hemanth.saratchandran@adelaide.edu.au}

\icmlkeywords{Machine Learning, ICML}

\vskip 0.3in
]



\printAffiliationsAndNotice{\icmlEqualContribution} 

\begin{abstract}
Modelling dynamical systems is an integral component for understanding the natural world. To this end, neural networks are becoming an increasingly popular candidate owing to their ability to learn complex functions from large amounts of data. Despite this recent progress, there has not been an adequate discussion on the architectural regularization that neural networks offer when learning such systems, hindering their efficient usage. In this paper, we initiate a discussion in this direction using coordinate networks as a test bed. We interpret dynamical systems and coordinate networks from a signal processing lens, and show that simple coordinate networks with few layers can be used to solve multiple problems in modelling dynamical systems, without any explicit regularizers. 
\end{abstract}

\section{Introduction}

Dynamical systems are systems whose state evolves over time. Modeling such systems from finite observations plays a major role in understanding, predicting, and controlling a vast array of physical and biological phenomena such as weather \cite{christensen2019reliable, knipp2016advances}, neuroscience \cite{izhikevich2007dynamical}, planetary motion \cite{koon2000dynamical, jiang2003bifurcation}, and molecular movements \cite{gorban2010principal, toni2010simulation}, among many others. Traditionally, analytical models derived from first principles played a key role in simulating dynamical systems. Nonetheless, a continuing concern of modeling dynamical system is that since measurements of physical quantities are typically obtained via sensors in the wild, they can often be noisy, irregular, and sparse. Since classical analytical tools for dynamical systems usually assume restrictive conditions, i.e., clean, regular, and relatively dense data, employing them on real-world data becomes less trivial. Further, although  analytical approaches usually benefit from explicit mathematical guarantees, their extrapolation to higher dimensional systems is often hindered by unrealistic conditions (e.g., sampling complexity) and impractical error bounds.

In contrast, there has been a recent surge of interest in using neural networks for modeling dynamical systems. With increasingly abundant computational resources and data, these methods have yielded impressive performance over classical simulation models. The underlying pillar of this success is the universal approximation properties of neural networks, that enables learning complex non-linear functions from data. Despite this trend, far too little  attention has been paid to the architectural regularization that neural networks implicitly offer in such settings. This lack of understanding obfuscates principled, efficient usage of neural networks in modeling dynamical systems, potentially leading to fairly involved architectures and explicit regularizers \cite{bakarji2022discovering, trischler2016synthesis, ku1995diagonal, chu2019adaptive, yeung2019learning}. Thus, this study strives to investigate the efficacy of implicit architectural bias that neural networks offer in the context of dynamical systems. To this end, we choose coordinate-networks as a test bed, a class of neural networks that is now ubiquitously being used across many computer vision tasks \cite{skorokhodov2021adversarial, chen2021learning, sitzmann2019scene, mildenhall2021nerf, li20223d, chen2022fully}. This choice is motivated by the simplicity of coordinate networks --- which are typically shallow fully connected networks --- and the strong architectural bias they offer when learning natural signals \cite{mildenhall2021nerf, tancik2020fourier}.  We begin with an interesting observation that both dynamical systems and coordinate-networks can be viewed through a signal processing lens. That is, a dynamical system can be  interpreted as a (multi-dimensional) signal that evolves over time. From this perspective,  modeling a dynamical system using finite measurements of physical quantities becomes analogous to recovering a multi-dimensional signal from discrete samples. On the other hand, we also note that coordinate-networks perform a similar task; given discrete coordinates and corresponding samples, coordinate-networks attempt to encode (reconstruct) a continuous signal. Inspired by this connection, we draw a parallel between dynamical systems and coordinate networks, and use sampling theory to bridge these two paradigms. First, we conduct a brief exposition of the Nyquist-Shannon sampling theory and show that coordinate-networks can be considered as generator functions under a generalized view of the former. Exploiting this insight, we propose a novel non-linear activation that allows coordinate networks to (theoretically) optimally reconstruct a given signal, while producing smooth first order derivatives of the network. With proper tuning of hyperparameters, our activation function enables controlling the bandwidth of the network, allowing the network to capture high-frequency dynamics while filtering noise. Then, we employ coordinate-networks across a series of physics problems, and demonstrate surprisingly improved, robust results compared to classical methods. It is important to note that across all considered problems, we only utilize the implicit bias of the neural architecture, omitting the need for explicit regularizers. 

Our contributions are as follows: 

\begin{itemize}
    \item We establish a parallel between dynamical systems and coordinate networks, and propose a novel activation function that performs better than existing activations. To the best of our knowledge, this work is the first to offer a theoretical comparison on the optimality of activation functions in reconstructing signals.

    \item We improve the results of the SINDY algorithm \cite{bakarji2022discovering} --- a method used to discover the governing equations of dynamical systems ---  by exploiting the smooth first order derivatives of coordinate-networks. We also demonstrate that our results are significantly robust to noise compared to the baseline.

    \item We show that coordinate networks implicitly learn the intrinsic rank of a given system from partial observations, eliminating the need for classical analytical tools such as time-delay embedding.

    \item We utilize coordinate networks for discovering the characteristics of  dynamical systems from partial observations. We further demonstrate that the recovered representations are extremely robust to random, sparse, and noisy measurements, compared to classical tools. 
    
    \item We demonstrate the efficacy of using coordinate networks for modeling higher dimensional systems, where the number of measurements are orders of magnitude lower than the optimal Nyquist rate.

\end{itemize}

\section{Preliminaries}

\subsection{Dynamical systems}

Dynamical systems can be defined in terms of a time dependant state space $\textbf{x}(t) \in \mathbb{R}^D$ where the time evolution of $\mathbf{x}(t)$ can be described via a differential equation,

\begin{equation}
\label{eq:dynamical_system}
    \frac{ d \mathbf{x}(t)}{dt} = f(\mathbf{x}(t), \alpha),
\end{equation}

where $f$ is a non-linear function and $\alpha$ are a set of system parameters. The solution to the differential equation \ref{eq:dynamical_system} gives the time dynamics of the state space $\textbf{x}(t)$. In practice, we only have access to discrete measurements $[\textbf{y}(t_1), \textbf{y}(t_2), \dots \textbf{y}(t_Q)]$ where $\textbf{y}(t) = g(\textbf{x}(t)) + \eta$ and  $\{t_n\}_{n=1}^{Q}$ are discrete instances in time. Here, $g(\cdot)$  can be the identity or any other non-linear function, and $\eta$ is noise. Thus, the central challenge in modeling dynamical systems can be considered as recovering the characteristics of the state space from such discrete observations.

\subsection{Coordinate networks}
Consider an  $L$-layer coordinate network, $F_L$, with widths $\{n_0,\ldots ,n_L\}$. The output at layer $l$, denoted $f_l$, is given by

\begin{equation}
\label{eq:coordinate_network}
f_l(x) = 
 \begin{cases}
            x, & \text{if}\ l = 0 \\
            \phi(W_lF_{l-1} + b_l), & \text{if}\ l \in [L-1] \\
            W_{L-1}F_{L-1} + b_L, & 
            \text{if}\ l = L
        \end{cases}
\end{equation}
where $W_l \in \R^{n_{l}\times n_{l-1}}$, $b_l \in \R^{n_l}$ are the weights biases respectively of the network, and $\phi$ is a non-linear activation function. Although the above formulation is identical to fully connected networks, they differ from traditional neural networks by usage. In contrast to the mainstream utilization of neural networks, where (very) high-dimensional inputs such as images, videos are mapped to a label space, coordinate networks are treated as a continuous data structure that \emph{encodes} a signal.  The inputs to coordinate networks are low-dimensional, discrete coordinates, \textit{e.g.},
$(x,y)$, and the outputs are samples of a particular signal at corresponding coordinates, \textit{e.g.} pixel intensities of an image sampled at $(x,y)$. The optimization minimizes the mean squared error (MSE) loss between the ground truth and the network predictions. In the above example, the coordinate network can be considered a continuous representation of an image, which can be queried up to extreme resolutions. Further, the activations used in coordinate networks determine their characteristics. For instance, ReLU activations have shown to suffer from spectral bias, hindering their performance in encoding high-frequency content, whereas recently proposed Gaussian \cite{rahaman2019spectral} and sinusoidal \cite{sitzmann2020implicit} activations allow high-fidelity signal reconstructions. It is also a common practice to use a positional encoding layer with coordinate networks, which modulates input coordinates with sin and cosine functions, capturing high-frequency content.

\section{Dynamical systems and coordinate networks}

We note that modeling dynamical systems and encoding signals using coordinate networks are analogous tasks. That is, modeling dynamical systems can be interpreted as recovering characteristics of a particular system via measured physical quantities over time intervals. Similarly, coordinate networks are used to recover a signal given discrete samples. Hence, in this section, we analyze coordinate networks from sampling theory based perspective, and propose an activation function for better signal reconstruction. Coordinate networks equipped with the newly proposed activation function will be validated on several problems in later sections.


\subsection{Revisiting sampling thoery}

The sampling theory concerns \emph{bandlimited} signals. A signal is bandlimited if, and only if, the magnitude of its Fourier spectrum is zero beyond a certain threshold frequency. More formally, let $f$ denote a continuous signal that is $\Omega$-band limited, meaning its Fourier transform $\widehat{f}(s) = 0$ for all $\vert s\vert > \Omega$.  If $f \in L^1(\R)$ is an $\Omega$-band limited signal, then the 
Nyquist-Shannon sampling theorem \cite{zayed2018advances} gives
\begin{equation}
	f(x) = \sum_{n=-\infty}^{\infty}f\bigg{(}\frac{n}{2\Omega} \bigg{)}
	sinc\bigg{(}2\Omega\big{(}x - \frac{n}{2\Omega} \big{)}\bigg{)}
\end{equation}
where the equality means converges in the $L^2$ sense. Thus, by sampling a signal at the lattice points 
$\frac{n}{2\Omega}$, for $n \in \Z$, and taking shifted sinc functions, it is possible to recover the signal provided we sample at a frequency of at least $2\Omega$-Hertz.
Theoretically, the theorem indicates that one would need an infinite number of samples for perfect reconstruction. This is, of course, not possible in practice. Further, it should be noted that the sampling theorem is an idealization of the real-world; natural signals are not always bandlimited. However, as natural signals tend to contain their dominant frequency modes at lower energies, one can project the original signal into a space of bandlimited functions with a finite dimension to get a good reconstruction. The sampling theory --- in its original form --- is only applicable to one-dimensional signals. However, it can be extended to higher dimensions in a straightforward manner. The main culprit of the sampling theory is the curse of dimensionality;  the exponential increase in the number of sample points needed to reconstruct a high-mode signal. This is a mathematical consequence of the fact that volumes of many mathematical shapes grow exponentially with dimension. This behaviour is detrimental for modeling dynamical systems using classical tools, e.g., dynamic mode decomposition (DMD), in higher dimensions, as the number of physical measurements required quickly becomes infeasible as the dimensions grow. In contrast, we show that coordinate networks can perform remarkably well in cases where the number of samples falls well below the Nyquist rate (see Sec.~\ref{sec:prediction}).

\subsection{Coordinate-networks for signal reconstruction}

In the previous section, we discussed how an exact reconstruction of a bandlimited signal could be achieved via a linear combination of shifted sinc functions. Thus, it is intriguing to explore if an analogous connection can be drawn to coordinate networks. In this picture, we fix a function $F$ and define a space $V(F)$ by
\begin{equation*}
 V(F) = \bigg{\{} s(x) = \sum_{k =-\infty}^{\infty}a(k)F(x-k) : a \in l^2  \bigg{\}},
\end{equation*}
where $l^2$ denotes the space of square summable sequences. In other words, a function $s \in V(F)$ is characterized by the sequence $a$, which is to be thought of as the discrete signal representation of $s$. As in the case of the sampling theorem, from the previous section, shifts of the function $F$ and the coefficients are enough to reconstruct $s$.

In order for the space $V(F)$ to be a good model to do signal processing, it is generally required that the functions $\{F_k = F(x - k)\}_{\Z}$ should form a Riesz basis, see section \ref{app:riesz} and \cite{unser2000sampling} for a detailed discussion on Riesz bases. The second requirement is that they satisfy the partition of unity condition
\begin{equation}\label{PUC}
 \sum_{k \in Z}F(x+k) = 1\text{, } \forall x \in \R.
\end{equation}
The partition of unity condition allows the space $V(F)$ to have the capability of approximating any input function arbitrarily close by selecting a sufficiently small sampling step. Thus it should be thought of as a generalisation of the Nyquist criterion in signal processing. We refer the reader to \cite{unser2000sampling} for details on how \eqref{PUC} determines a Nyquist-type sampling criterion. Interestingly, we observe that coordinate networks implicitly perform a similar task of reconstructing a signal via shifted basis functions, as discussed next. 

Consider a coordinate network formulated as Eq.~\ref{eq:coordinate_network}. Let 
$X \in \R^{n_0\times N}$ and $Y \in \R^{n_L \times N}$ denote the input coordinates and the corresponding signal measurements, where $N$ denotes the number of samples. For $X = [x_1,\ldots ,x_N]^T\in \R^{N\times n_0}$, the feature matrix at layer $l$ is given by $F_l = [f_l(x_1),\ldots ,f_l(x_N)]^T \in \R^{N\times n_l}$. 
For simplicity, now consider a 2-layer coordinate network with one-dimensional output. Such a network can be expressed by
\begin{equation}\label{NN_sampling_view}
F(X) = W_2\phi(W_1X + b_1) + b_2.
\end{equation} 
This equation exhibits the interesting observation that the output function is being constructed via shifted samples of the non-linearity. When training, the network seeks to optimise the weights and biases so as to fit the training labels optimally. In this way, the training can be thought of as trying to pick suitable shifts and bandwidth of the non-linearity that can reconstruct the output signal. However, note that in the previous sections, the sampling studied has all been uniform sampling. In general, sampling theory must assume uniform sampling, for otherwise, the methods of Fourier analysis cannot be used to get a signal reconstruction theorem. While there have been various algorithms to tackle non-uniform sampling for signal reconstruction, one of the highlight points of reconstructing with a coordinate network is that it has no issues with non-uniform samples (see Sec.~\ref{sec:dynamics}). 


\begin{figure*}[!htp]
    \centering
    \includegraphics[width=1.9 \columnwidth]{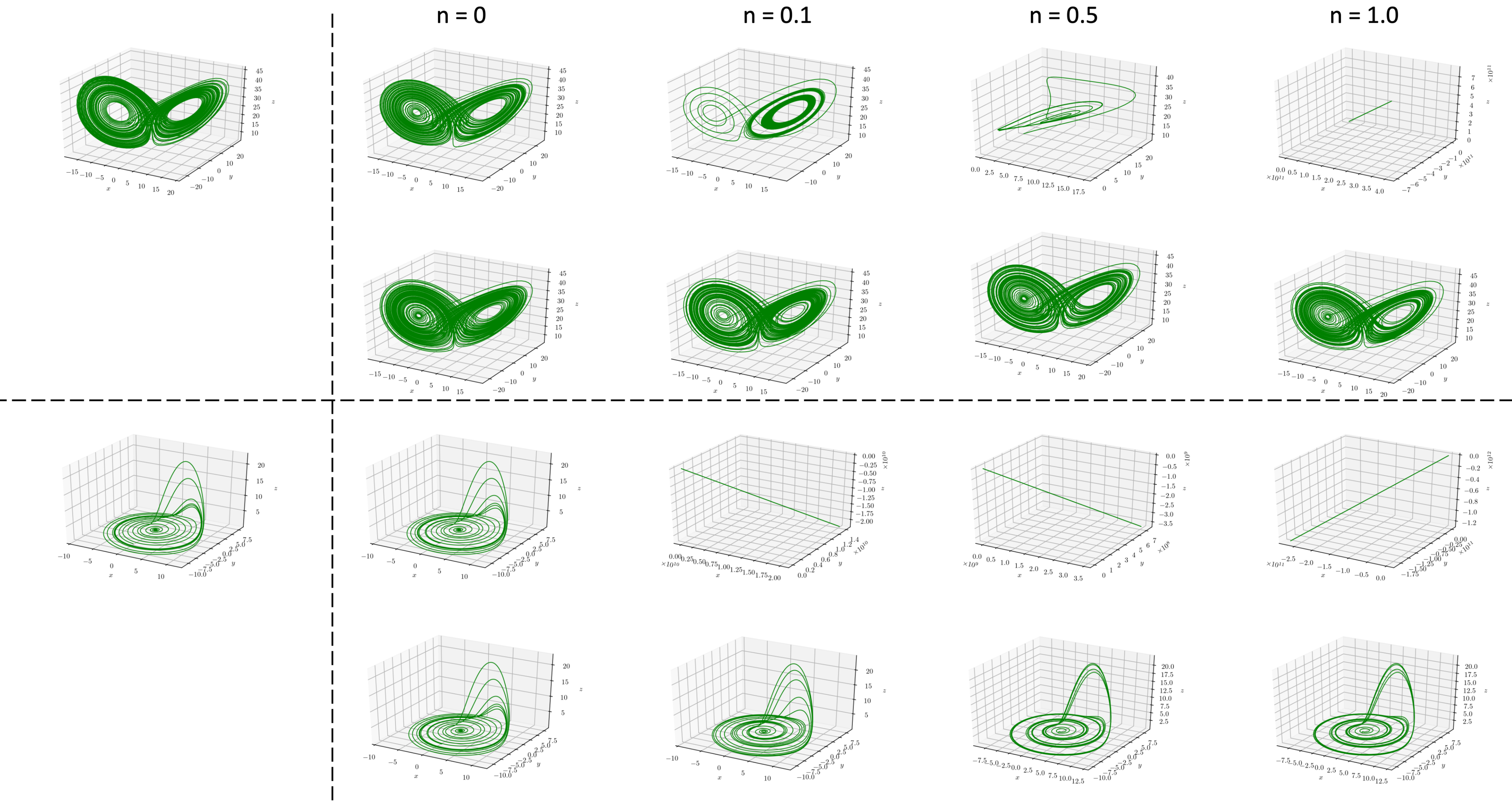}
    \vspace{-2em}
    \caption{ We use coordinate networks to improve the results of the SINDy algorithm. The top block and the bottom block demonstrate experiments on the Lorenz system and the Rossler system, respectively. In each block, the top row and the bottom row represent the results of the baseline SINDy algorithm and the improved version (using coordinate networks). As evident, coordinate networks can be used to obtain significantly robust results.}
    \label{fig:sindy}
\end{figure*}

\textbf{Activations.} The discussion thus far asserted that encoding signals using coordinate networks can be considered as reconstructing signals with shifted basis functions, which are essentially the type of activation functions used. One can immediately see that, from a theoretical perspective, using the sinc function as the activation should be optimal for reconstructing signals. Indeed, we utilize sinc as the activation function in  our experiments and observe better performance compared to previously proposed Gaussian or Sinusoidal activations (Fig. \ref{fig:sindyq}). We speculate that the sinc activations can potentially improve the results across many other computer vision tasks where coordinate networks are used. Nonetheless, we limit our experiments to dynamical systems in this work. On the theoretical side, what makes sinc an optimal signal reconstructor is the fact that it generates a Riesz basis that satisfies the \textit{partition of unity} condition \eqref{PUC} \cite{unser2000sampling}. 
In \cref{app:gaussian reconstruct}, we show that Gaussian functions generate a Riesz basis but only satisfy the partition of unity condition \eqref{PUC} approximately, making them inferior to the sinc function for sampling. We also show that sinusoidal functions are only optimal for sampling periodic signals, see \cref{app:sinusoid}. Finally, we show that the
 ReLU function  does \emph{not} generate a Riesz basis (proposition \ref{relu_riesz_details}) and catastrophically fail to satisfy \eqref{PUC}, see \cref{app:relu}. However, it should be noted that this result does not undermine the implicit regularization properties of ReLU activations in recovering low-frequency signals.

\subsection{Lipschitz constant of coordinate networks}

A major obstacle to modelling dynamical systems is the noisy measurements. Noise typically gets amplified through non-linear systems, exhibiting deleterious effects on classical analytical tools. On the other hand, we observe that by controlling the hyparparameters of the sinc function, it is possible to smoothen the signal encoded by coordinate networks, making them act as implicit noise filters (see Fig. \ref{fig:omega}). To explain this behavior, we show below that the hyperparameters of the sinc activation can manipulate the Lipschitz constant of the network. We also offer an equivalent result for Gaussians in the Appendix (theorem \ref{layer_lipshitz_gauss}).




\begin{theorem}\label{layer_lipshitz_sinc}
Let $f_L$ denote a neural network emplying a sinc non-linearity, 
$sinc(\omega x) = \frac{sin(\omega x)}{\omega x}$. Then the Lipshitz constant of $f_k$, for $1 \leq k \leq L$ increases as 
$\omega$ increases. In other words, increasing $\omega$ increases the kth-layers Lipshitz constant.
\end{theorem}

(Proof in \ref{layer_lipshitz_sinc_proof}). Further, we assert that the Lipschitz constant of a network with sinc activation is inherently linked to the rank of the hidden layer representations. This provides an important, controllable architectural bias (Fig.\ref{fig:omega}).

\begin{theorem}
Let $f_L$ denote a coordinate neural network with activation 
$sinc(\omega x) = \frac{sin(\omega x)}{\omega x}$. Furthermore, fix a training data set $X = \{x_i\}$ sampled from a fixed training distribution $\mathcal{P}$.
Then increasing $\omega$ leads to on average an increase in the stable rank of the feature map $F_l$.
\end{theorem}

(Proof in \ref{eq:stable rank}).



\section{Discovering governing equations}

SINDy algorithm aims to recover the governing equations of a dynamical system from discrete observations of underlying variables. Given a set of samples $\mathbf{Y} = [\mathbf{y}(t_1), \mathbf{y}(t_2), \dots \mathbf{y}(t_N) ] \in \mathbb{R}^{D \times N}$, SINDY computes $\dot{\mathbf{Y}} = [\dot{\mathbf{y}(t_1)}, \dot{\mathbf{y}(t_2)}, \dots \dot{\mathbf{y}(t_N)} ] \in \mathbb{R}^{D \times N}$ using a finite difference based or continuous approximation technique. Then, a library of $Q$ candidate functions are assumed to build the library matrix $\Theta (\textbf{Y}) = [ \theta_1(\textbf{\textbf{Y}}, ), \theta_2(\textbf{\textbf{Y}}, ), \dots, \theta_Q(\textbf{\textbf{Y}}, ) ]$. Finally, SINDY minimizes the loss,

\begin{equation}
    L_{S} = ||{\dot{\textbf{Y}} - \Theta (\textbf{Y})\Gamma}||^2_2 + \lambda ||{\Gamma}||^2_1,
\end{equation}

\begin{figure*}[!htp]
    \centering
    \includegraphics[width=1.9\columnwidth]{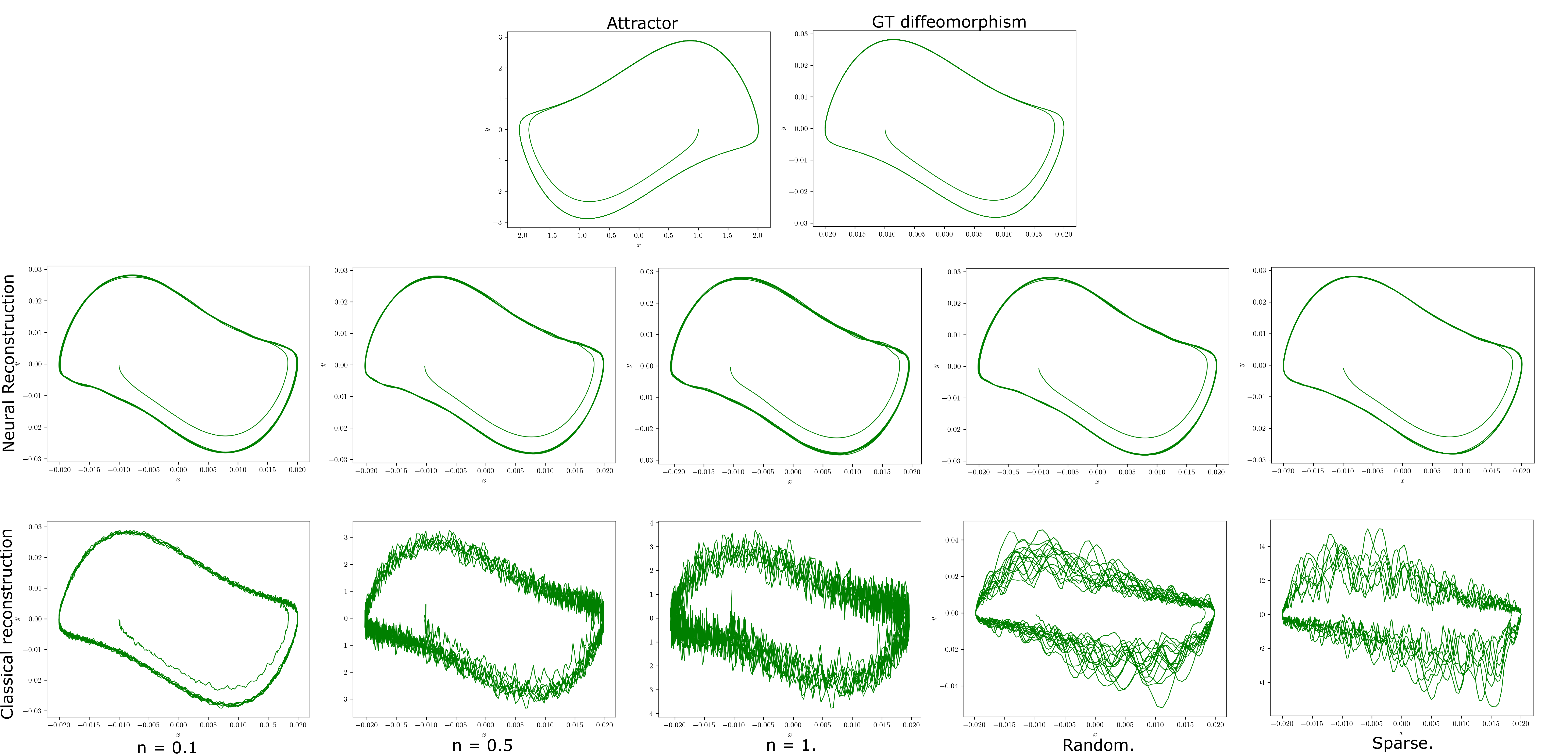}
    \vspace{-1em}
    \caption{Discovering the dynamics from partial observations. We use the Vanderpol system for this experiment. \textit{Top row:}  the original attractor and the diffeomorphism obtained by the SVD decomposition of the Hankel matrix (see Sec.~\ref{sec:dynamics}) without noise. \textit{Second row:} The same procedure is used to obtain the reconstructions with noisy, random, and sparse samples. \textit{Third row:} First, a coordinate network is used to obtain a continuous reconstruction of the signal from discrete samples, which is then used as a surrogate signal to resample measurements. Afterwards, the diffeomorphisms are obtained using those measurements. As shown, coordinate networks are able to recover the dynamics more robustly with noisy, sparse, and random samples. }
    \label{fig:diffeomorphism}
\end{figure*}

where $\Gamma$ is a sparsity matrix that choose candidate functions from $\Theta$ while enforcing sparsity. We note that coordinate networks offer two forms of important architectural biases here; suppose we train a coordinate network using $\{t_n\}^Q_{n=1}$ and $\textbf{Y}$ as inputs and labels, respectively, to reconstruct a continuous representation of $\textbf{Y}$. \textit{a}) by controlling $\omega$ of sinc functions (while training), one can filter high-frequency noise embedded in $\textbf{Y}$ and \textit{b}) It is possible to obtain measurements $\dot{\textbf{Y}}$ by computing the Jacobian of the network, utilizing smooth derivatives of sinc-activated coordinate networks. In comparison, we observe that ReLU activations yield inferior results, possibly due to the noisy first-order derivatives caused by their piece-wise linear approximation of functions \ref{fig:sindyq}. We perform an experiment to demonstrate the efficacy of these architectural biases below.

\textbf{Experiment 1: } We use the Lorenz system and the Rossler system (see \ref{ds_eqns}) for this experiment. We obtain $1000$ samples between $0$ and $100$ with an interval of $0.1$ to create $\textbf{Y}$. Then, we inject noise to $\textbf{Y}$ from a uniform distribution $\eta \sim U(-n, n)$ by varying $n$. As the baseline, for each noise scale, we used spectral derivatives to compute $\dot{\textbf{Y}}$. Note that we empirically chose spectral derivatives to obtain the best baseline after comparing other alternatives to compute $\dot{\textbf{Y}}$, including finite difference methods and polynomial approximations. As the competing method, again for each noise scale, we use a coordinate network to reconstruct a continuous signal by training the network on $\textbf{Y}$ samples. Afterwards, we compute the Jacobian of the network to compute $\dot{\textbf{Y}}$ on the same coordinates. Then, we minimize $L_S$ using the the computed $\textbf{Y}$ and $\dot{\textbf{Y}}$. Then, for both cases, we utilize the SINDy algorithm to obtain the governing equations of each system. The dynamics recovered from the discovered equations are compared in Fig. \ref{fig:sindy}. As evident, using coordinate networks for this particular task yields surprisingly robust results at each noise scale, compared to the baseline. We use a $4$-layer sinc-activated coordinate network for this experiment, where the width of each layer is $256$.

\begin{figure}[!htp]
    \centering
    \includegraphics[width=0.9\columnwidth]{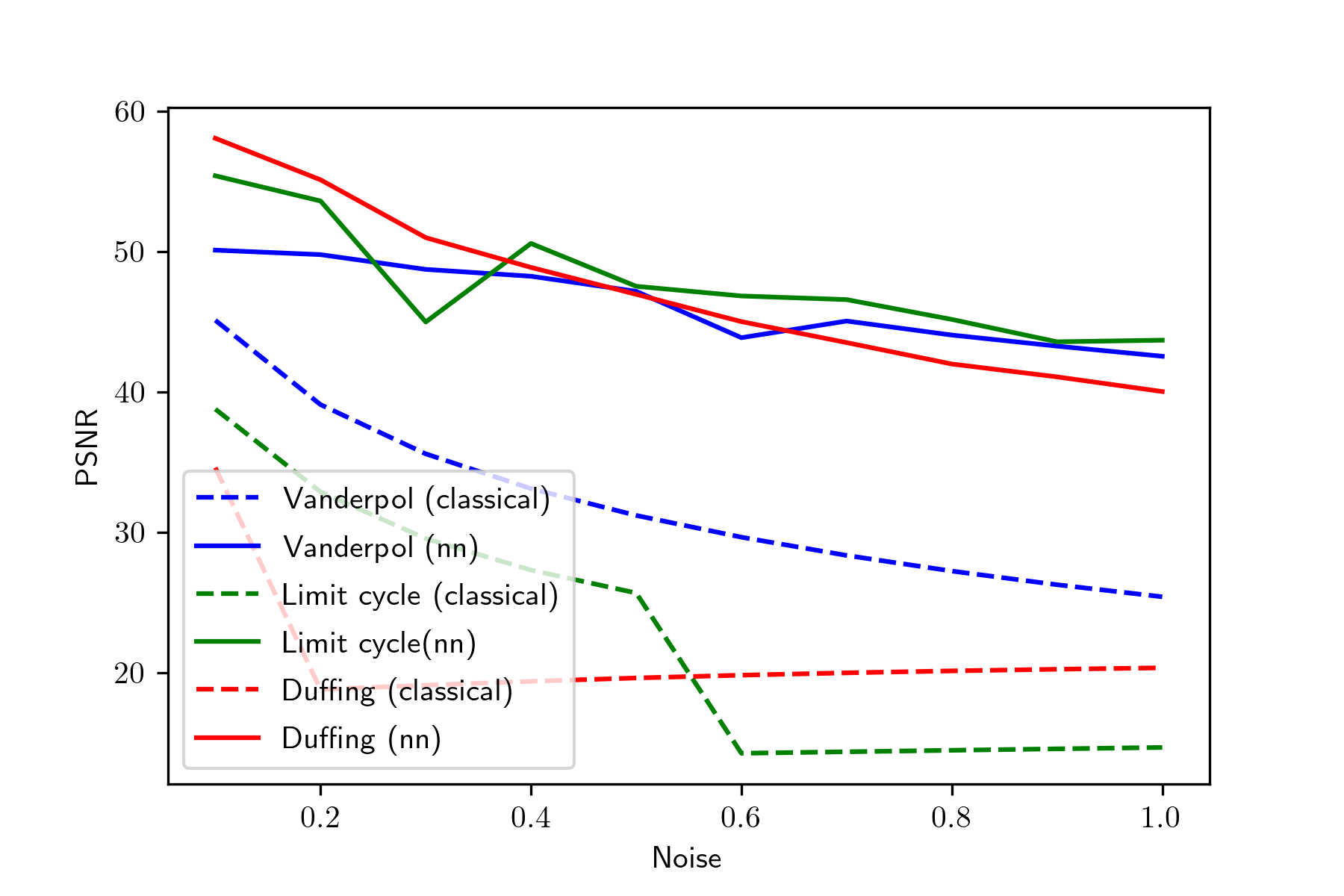}
    \vspace{-2.em}
    \caption{Discovering dynamics with partial observations. Reconstructions across several systems are compared.  }
    \label{fig:diff_q}
\end{figure}

\section{Mode discovery from partial observations.}
\label{sec:mode discovery}
\vspace{-0.5em}
Natural systems typically depend on multiple underlying variables (e.g., temperature, pressure, velocity etc.). The number of such factors is also known as the \emph{modes} or the \emph{intrinsic rank} of the system. However, in practical settings, the number of modes of a system is not apriori known, and only a subset of variables are measured. In such scenarios, ``time-delay-embeddings" provide an analytical tool to determine the number of modes of a system by \emph{only} observing the dynamics of a subset of variables. Remarkably, we show that the coordinate networks can be used to achieve the same task, \emph{without} time delay embedding. Further, we depict that this neural mode recovery is robust to higher dimensions, whereas classical time-delay-embedding fails in such cases. 

\begin{figure}[!htp]
    \centering
    \includegraphics[width=1.\columnwidth]{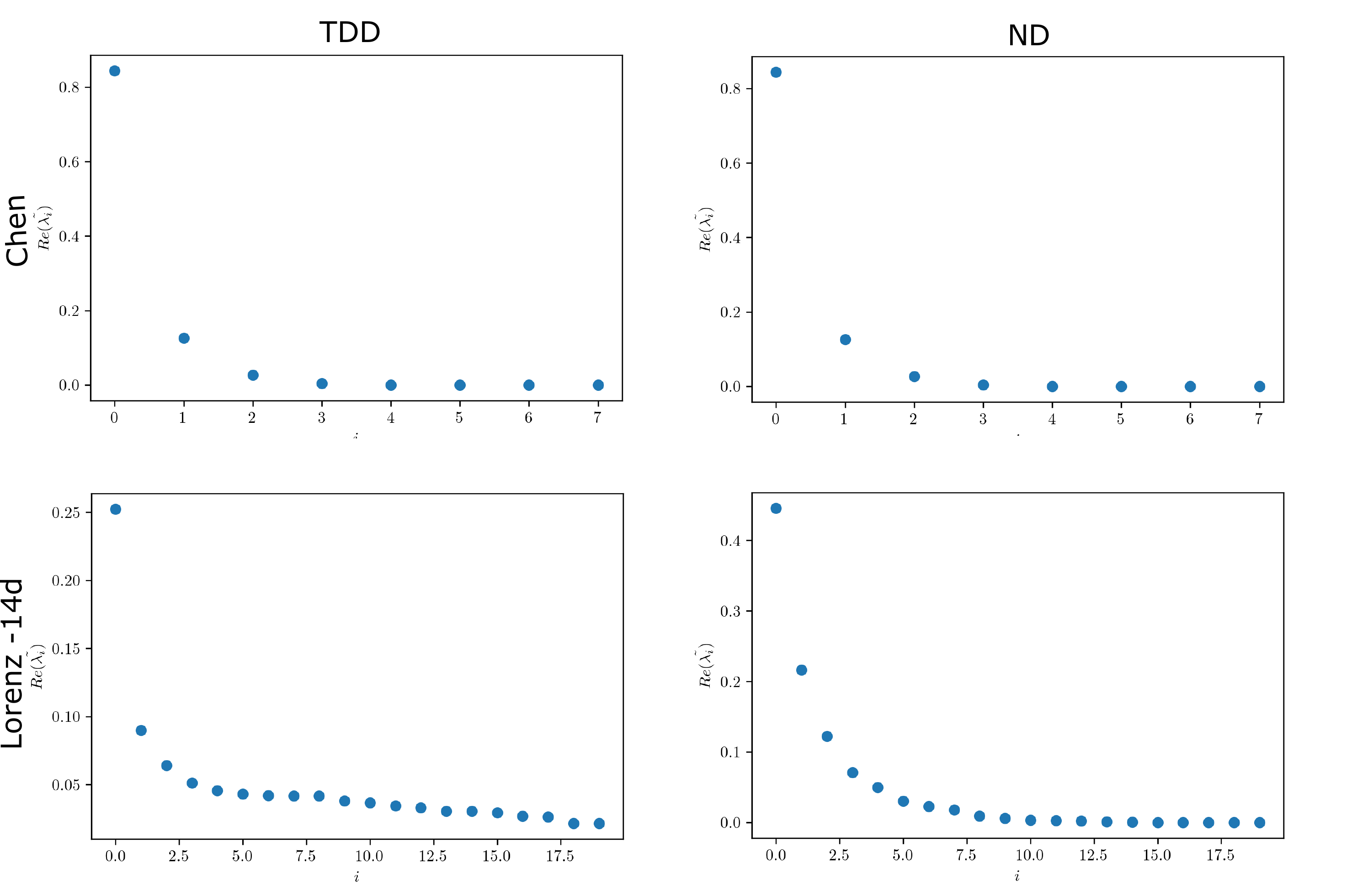}
    \vspace{-2em}
    \caption{ We compare time delay decomposition (TDD) and neural decomposition (ND) in discovering the modes of a dynamical system from partial observations. Remarkably, we find that coordinate networks can be used to discover the hidden modes of the system \emph{without} time delay embedding. In the Chen system, both methods perform well in identifying the three modes. In the 14-dim Lorenz system, TDD fails, while ND is able to identify 13 modes by only observing the dynamics of a single variable. }
    \label{fig:mode discovery}
\end{figure} 

\textbf{Time delay embedding.} Given a set of discrete samples of the observable variable 
 $[y_1(t_1), y_1(t_2), \dots, y_1(t_Q)]$, a Hankel matrix $\textbf{H}$ can be created by augmenting the samples as delay embeddings in each row:

 \begin{equation}
 \label{eq:hankel}
     \textbf{H} = \begin{bmatrix}
     y_1(t_1) & y_1(t_2) & \dots & y_1(t_n) \\
     y_1(t_2) & y_1(t_3) & \dots & y_1(t_{n+1}) \\
     \vdots & \vdots & \ddots & \vdots \\
     y_1(t_m) & y_1(t_{m+1}) & \dots & y_1(t_{m+n+1})
     \end{bmatrix}.
 \end{equation}

Then, eigen values of $\textbf{H}$ can be obtained by the SVD decomposition $\textbf{H} = \mathbf{U\Sigma V}^T$. The number of dominant eigenvalues can be used to identify the intrinsic rank of the system.

In contrast, we train a coordinate network on  $[y_1(t_1), y_1(t_2), \dots, y_1(t_Q)]$ to obtain a continuous representation of the signal. Let the mapping from the inputs $t$ to the penultimate layer be $\phi(t):\mathbb{R} \to \mathbb{R}^K$, where the penultimate layer is $K$-dimensional. Then, we extract the penultimate layer outputs $\Psi = [\phi(t_1)^T, \phi(t_2)^T, \dots, \phi(t_Q)^T] \in \mathbb{R}^{K \times Q}$, and perform SVD on $\Psi$ to extract singular values. Remarkably, we found that the number of dominant (non-zero) singular values is equal to the intrinsic rank of the system. This observation suggests that the architectural bias of coordinate networks implicitly leads to learning a Koopman basis in the penultimate layer. However, we limit the scope of this work to this empirical observation and leave a thorough theoretical exploration to future work.  

\textbf{Experiment 2: }We use a three-dimensional Chen system and a 14-dim Lorenz system (\ref{ds_eqns}) for this experiment. As a baseline, we perform time-delay decomposition on both signals. Then, we perform neural decomposition using a 4-layer coordinate network with $20$-width layer. The results are shown in Fig.~\ref{fig:mode discovery}. As depicted, the baseline fails to correctly identify the modes of the signal in higher dimensions, while the neural decomposition identifies $13$ modes correctly.  

\vspace{-0.5em}
\section{Discovering the dynamics of latent variables}

\label{sec:dynamics}

In Sec.\ref{fig:mode discovery}, we discussed how coordinate networks could be utilized to discover the number of modes (latent variables) of a dynamical system by observing the dynamics of a single variable. When only such partial measurements are available, it is generally not possible to derive a closed-form model of the system. However, Taken's theorem (\ref{sec:taken}) states that under certain conditions, it is possible to augment the partial measurements as delay embeddings, which yields an attractor that is diffeomorphic to the original attractor. This is an extremely powerful tool, as it enables discovering certain dynamics of a complex system by only observing a handful of variables. The procedure is as follows: First,   a Hankel matrix is computed as in Eq.~\ref{eq:hankel}. Then, the eigenvectors that span the time delay space are obtained by performing SVD decomposition on $\mathbf{H}$. As per Taken's theorem, one can now obtain a diffeomorphism of the original attractor via the dominant eigenvectors of the time delay space.


Nonetheless, the above procedure should adhere to restrictive conditions; a) the measurements should be equally spaced, and b) the intervals between measurements have to adhere to the condition $n \tau \approx 0.1$, where $n$ is the width of the Hankel matrix and $\tau$ is the time interval between two samples. Further, as we demonstrate, the obtained dynamics (diffeomorphisms) are extremely sensitive to noise. On the contrary, using a coordinate network to encode the original measurements as a continuous signal, and then using the coordinate network as a surrogate signal to create the Hankel matrix leads to surprisingly robust results. Further, the continuous reconstruction we get from the coordinate network requires sparser samples ($n \tau = 0.2)$), overriding a  restrictive condition. 

\textbf{Experiment 3: } We use a Vanderpol system, Limit cycle attractor, and Duffing equation for this experiment \ref{ds_eqns}. We use $5000$ samples, sampled between $0-100$, to create the Hankel matrix. The results are illustrated in Fig.~\ref{fig:diffeomorphism} and Fig.~\ref{fig:diff_q}. To demonstrate the effect of noise on recovered dynamics, we add uniformly sampled noise to the Hankel matrix. In the sparse sampling scenario, we increase the sampling interval by a factor of two. As evident, coordinate networks are able to produce significantly robust results. In other words, this enables one to accurately recover the dynamics of a system with partial observations that are noisy, random, and sparse, whereas the performance of the classical method degrades in each case. 

\section{Predicting the future states}
\label{sec:prediction}

\begin{figure}[!htp]
    \centering
    \includegraphics[width=1.\columnwidth]{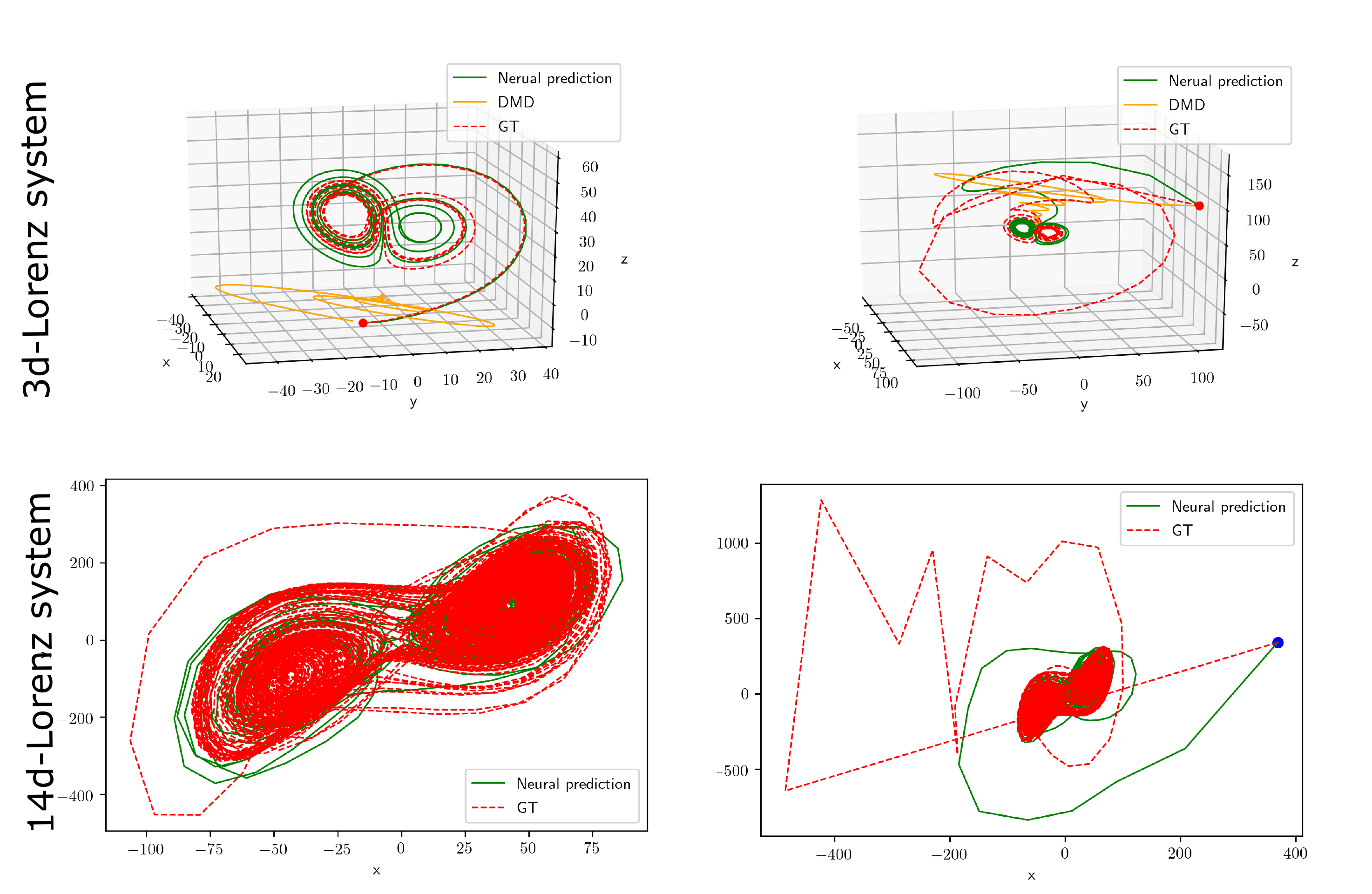}
    \vspace{-2em}
    \caption{Interpolation and extrapolation of coordinate networks. The training points are only sampled within the globes of the attractors. When an initial test point is randomly sampled within the bounds of the training data (first column), the coordinate network is able to follow the ground truth trajectories a considerable amount of time into the future, while DMD fails. Note that in the 14-dimensional case, we project the dynamics into two variables for visualization. In this case, the ratio of training samples to the Nyquist rate is extremely small ($\approx 1.81 \times 10^{-8}$), but still, the coordinate network shows satisfactory performance, demonstrating its extreme sampling efficiency. The second column illustrates the extrapolation capability of the coordinate networks. Even when the initial sampling point is far away from the training data, the network still is able to converge to the attractor.  }
    \label{fig:dmd}
\end{figure} 

The time evolution of a dynamical system can be learned by
moeling the relationship $x_{t+1} = \gamma (x_t)$, where $\gamma$ is a non-linear function, which can also be thought of as time-series forecasting. A popular tool for obtaining these forecasting dynamics is Dynamic Mode Decomposition (DMD), which has a strong connection to vector autoregressive models (VAR). Given a set of $D$ dimensional data points, DMD attempts to learn the linear transformation $A$ such that $\mathbf{x}_{t+1} \approx \mathbf{A} (\mathbf{x}_t)$. One can immediately see that this is a linear approximation of the system. $\mathbf{A}$ can be computed via finding a sparse solution to the equation $\textbf{X}_2 = \textbf{A}\textbf{X}_1$, where $\textbf{X}_1 = [\mathbf{x}(t_1), \mathbf{x}(t_2), \dots , \mathbf{x}(t_Q)]$ and $\textbf{X}_1 = [\mathbf{x}(t_2), \mathbf{x}(t_3), \dots , \mathbf{x}(t_{Q+1})]$. For a more comprehensive read on DMD, we refer the reader to \cite{tu2013dynamic}.
However, critical limitations of DMD include the linearization of the system and, most importantly, the required sampling density. For ideal reconstruction, DMD has to comply with the Nyquist sampling rate \cite{fathi2018applications}. This becomes a significant drawback when forecasting with higher-dimensional signals, as the Nyquist sampling rate increases exponentially with the number of dimensions \ref{sec:higher dimensions}. We note that this is a common drawback to \emph{any} classical forecasting system, not only DMD. We also point out that there are clever workarounds to this problem that assumes the sparsity of data in some basis, \textit{e.g.}, compressed sensing \cite{brunton2013compressive}. However, such methods demand domain knowledge of the system and often include fairly involved mathematical modeling.  

Due to the above reason, forecasting the dynamical systems using neural networks has attracted a significant amount of interest recently \cite{park2022recurrent, isamiddin2021development}. These methods can handle high dimensional data more effectively and can learn complex, non-linear functions. These methods utilize the explicit recurrent relationships built into the neural architectures to obtain predictive dynamics. In contrast, we show that by  using coordinate networks, which are simply fully connected networks that minimizes an MSE objective, it is possible to enjoy impressive predictive performances.

\textbf{Experiment 4: } For this experiment, we use a three-dimensional Lorenz system and its generalization to a 14-dimensional system. In both cases,  we create a set of $\textbf{X}_1$ and $\textbf{X}_2$ matrices using snapshots of random trajectories, starting from random initial points. For the 3-dimensional system and the 14-dimensional system, we use $20$ and $100$ trajecories, respectively. We take $800$ snapshots of each trajectory, using $0.01$ time intervals. Afterwards, we train one coordinate network per each system by feeding the columns of $\textbf{X}_1$ as inputs, and providing the corresponding columns of $\textbf{X}_2$ as labels. Thus, the network learns to predict the future state after a fixed time interval, given the current state. After training, We test the coordinate networks for two scenarios. 1) We randomly sample a point in the space,  within the bounds of the training data, and recursively use the coordinate network to advance into future states. We computed the average frequency across each axis of the 14-dimensional system to be approximately $8$, which means the Nyquist sampling rate is $8^{14} \approx 4.4 \times 10^{12}$. However, we only use $800000$ samples for training the network, where the ratio to the Nyquist rate is $1.81 \times 10^{-8}$. Nonetheless, we still managed to capture the dynamics of the system satisfactorily (Fig.~\ref{fig:dmd}). This illustrates the extreme sampling efficiency of coordinate networks. 2) Usually, neural networks are known to demonstrate weak generalizations to out-of-distribution data. However, we found that coordinate networks yield surprisingly good extrapolations; when a random point is sampled far from the training data distribution, the network still managed to converge to the attractor.








\section{Related Work}

\textbf{Data driven dynamical systems modeling.} There have been several works that have undertaken a study of data driven discovery of dynamical systems using a variety of techniques such as, nonlinear regression \cite{voss1999amplitude}, 
empirical dynamical modelling \cite{ye2015equation}, normal form methods \cite{majda2009normal}, spectral analysis \cite{giannakis2012nonlinear}, Dynamic mode decomposition (DMD) \cite{schmid2010dynamic, kutz2016dynamic}. Compressed sensing and sparse regression in a library of candidate models has also been used to identify dynamical systems
\cite{reinbold2021robust, wang2011predicting, naik2012nonlinear, brunton2016discovering}, 
\cite{tran2017exact} in the context of corrupted data. Reduced modelling techniques have also been widely used in the analysis of dynamical systems such as, proper orthogonal  decomposition (POD) \cite{holmes2012turbulence, kirby2001geometric, sirovich1987turbulence, lumley1967structure}, local and global POD methods 
\cite{schmit2004improvements, sahyoun2013local}, adaptive POD methods \cite{singer2009using, peherstorfer2015online}. DMD methods with Koopman operator theory
\cite{budivsic2012applied, mezic2013analysis} have also been used for system identification. Neural networks have been used for system identification and discovery of governing equations \cite{qin2019data, gonzalez1998identification, lecun2015deep, chen2018neural, jaeger2004harnessing, raissi2019physics, lu2021learning}. For time series forecasting, recurrent neural networks (RNNs)
\cite{bailer1998recurrent, uribarri2022dynamical} and long short-term memory networks 
(LSTM) \cite{graves2012long, wang2011predicting} are the most commonly used networks. Data driven discovery using deep neural networks was carried out in \cite{qin2019data}, and convolution neural networks have also found applications in system identification
\cite{mukhopadhyay2020learning}.

\textbf{Coordinate networks.} Coordinate networks are a recently popularized class of fully connected neural networks by the seminal work of \cite{mildenhall2021nerf}. Although, in principle, any activation function can be used with coordinate networks, traditional activations such as ReLU, Sigmoid etc. tend to suffer from spectral bias \cite{rahaman2019spectral}, hindering their ability to learn high-frequency content. As a workaround, \cite{mildenhall2021nerf} employed a positional embedding layer to project the inputs to a higher dimensional space, which allowed the network to model high frequencies more effectively. Further,  Sitzmann et al. \cite{sitzmann2020implicit} proposed a sinusoidal activation, known as SIREN, which eliminated the need for positional embedding layers. However, SIREN exhibits volatile performance against random initializations. In contrast, \cite{ramasinghe2022beyond} introduced a Gaussian activated coordinate MLP that, like SIREN, showed state-of-the-art performance on signal reconstruction. An advantage of Gaussian activated coordinate MLPs is 
that they are robust to random initialisation schemes such as Xavier Uniform, and Xavier Normal. Nonetheless, to the best of our knowledge,  so far, there has not been a discussion on the theoretical optimality of these activations for reconstructing signals. In contrast, we analyze these activations from a signal-processing perspective and propose a novel activation function.

\section{Conclusion}

In this work, we explore the efficacy of using the implicit architectural regularization of coordinate networks for modeling dynamical systems. We propose a novel activation function that is better suited for reconstructing signals, and utilize it across several problems. Notably, we use relatively shallow (4-layer) networks for our evaluations and achieve robust and improved results compared to traditional tools, without explicit regularizers.

\bibliography{example_paper}

\begin{thebibliography}{67}
\providecommand{\natexlab}[1]{#1}
\providecommand{\url}[1]{\texttt{#1}}
\expandafter\ifx\csname urlstyle\endcsname\relax
  \providecommand{\doi}[1]{doi: #1}\else
  \providecommand{\doi}{doi: \begingroup \urlstyle{rm}\Url}\fi

\bibitem[Bailer-Jones et~al.(1998)Bailer-Jones, MacKay, and
  Withers]{bailer1998recurrent}
Bailer-Jones, C.~A., MacKay, D.~J., and Withers, P.~J.
\newblock A recurrent neural network for modelling dynamical systems.
\newblock \emph{network: computation in neural systems}, 9\penalty0
  (4):\penalty0 531, 1998.

\bibitem[Bakarji et~al.(2022)Bakarji, Champion, Kutz, and
  Brunton]{bakarji2022discovering}
Bakarji, J., Champion, K., Kutz, J.~N., and Brunton, S.~L.
\newblock Discovering governing equations from partial measurements with deep
  delay autoencoders.
\newblock \emph{arXiv preprint arXiv:2201.05136}, 2022.

\bibitem[Brunton et~al.(2013)Brunton, Proctor, and
  Kutz]{brunton2013compressive}
Brunton, S.~L., Proctor, J.~L., and Kutz, J.~N.
\newblock Compressive sampling and dynamic mode decomposition.
\newblock \emph{arXiv preprint arXiv:1312.5186}, 2013.

\bibitem[Brunton et~al.(2016)Brunton, Proctor, and
  Kutz]{brunton2016discovering}
Brunton, S.~L., Proctor, J.~L., and Kutz, J.~N.
\newblock Discovering governing equations from data by sparse identification of
  nonlinear dynamical systems.
\newblock \emph{Proceedings of the national academy of sciences}, 113\penalty0
  (15):\penalty0 3932--3937, 2016.

\bibitem[Budi{\v{s}}i{\'c} et~al.(2012)Budi{\v{s}}i{\'c}, Mohr, and
  Mezi{\'c}]{budivsic2012applied}
Budi{\v{s}}i{\'c}, M., Mohr, R., and Mezi{\'c}, I.
\newblock Applied koopmanism.
\newblock \emph{Chaos: An Interdisciplinary Journal of Nonlinear Science},
  22\penalty0 (4):\penalty0 047510, 2012.

\bibitem[Chen et~al.(2022)Chen, Kwiatkowski, Vondrick, and
  Lipson]{chen2022fully}
Chen, B., Kwiatkowski, R., Vondrick, C., and Lipson, H.
\newblock Fully body visual self-modeling of robot morphologies.
\newblock \emph{Science Robotics}, 7\penalty0 (68):\penalty0 eabn1944, 2022.

\bibitem[Chen et~al.(2018)Chen, Rubanova, Bettencourt, and
  Duvenaud]{chen2018neural}
Chen, R.~T., Rubanova, Y., Bettencourt, J., and Duvenaud, D.~K.
\newblock Neural ordinary differential equations.
\newblock \emph{Advances in neural information processing systems}, 31, 2018.

\bibitem[Chen et~al.(2021)Chen, Liu, and Wang]{chen2021learning}
Chen, Y., Liu, S., and Wang, X.
\newblock Learning continuous image representation with local implicit image
  function.
\newblock In \emph{Proceedings of the IEEE/CVF conference on computer vision
  and pattern recognition}, pp.\  8628--8638, 2021.

\bibitem[Christensen \& Berner(2019)Christensen and
  Berner]{christensen2019reliable}
Christensen, H.~M. and Berner, J.
\newblock From reliable weather forecasts to skilful climate response: A
  dynamical systems approach.
\newblock \emph{Quarterly Journal of the Royal Meteorological Society},
  145\penalty0 (720):\penalty0 1052--1069, 2019.

\bibitem[Chu et~al.(2019)Chu, Fei, and Hou]{chu2019adaptive}
Chu, Y., Fei, J., and Hou, S.
\newblock Adaptive global sliding-mode control for dynamic systems using double
  hidden layer recurrent neural network structure.
\newblock \emph{IEEE transactions on neural networks and learning systems},
  31\penalty0 (4):\penalty0 1297--1309, 2019.

\bibitem[Fathi(2018)]{fathi2018applications}
Fathi, M.~F.
\newblock \emph{Applications of Dynamic Mode Decomposition and Sparse
  Reconstruction in The Data-Driven Dynamic Analysis of Physical Systems}.
\newblock PhD thesis, The University of Wisconsin-Milwaukee, 2018.

\bibitem[Giannakis \& Majda(2012)Giannakis and Majda]{giannakis2012nonlinear}
Giannakis, D. and Majda, A.~J.
\newblock Nonlinear laplacian spectral analysis for time series with
  intermittency and low-frequency variability.
\newblock \emph{Proceedings of the National Academy of Sciences}, 109\penalty0
  (7):\penalty0 2222--2227, 2012.

\bibitem[Gonz{\'a}lez-Garc{\'\i}a et~al.(1998)Gonz{\'a}lez-Garc{\'\i}a,
  Rico-Mart{\`\i}nez, and Kevrekidis]{gonzalez1998identification}
Gonz{\'a}lez-Garc{\'\i}a, R., Rico-Mart{\`\i}nez, R., and Kevrekidis, I.~G.
\newblock Identification of distributed parameter systems: A neural net based
  approach.
\newblock \emph{Computers \& chemical engineering}, 22:\penalty0 S965--S968,
  1998.

\bibitem[Gorban \& Zinovyev(2010)Gorban and Zinovyev]{gorban2010principal}
Gorban, A.~N. and Zinovyev, A.
\newblock Principal manifolds and graphs in practice: from molecular biology to
  dynamical systems.
\newblock \emph{International journal of neural systems}, 20\penalty0
  (03):\penalty0 219--232, 2010.

\bibitem[Graves \& Graves(2012)Graves and Graves]{graves2012long}
Graves, A. and Graves, A.
\newblock Long short-term memory.
\newblock \emph{Supervised sequence labelling with recurrent neural networks},
  pp.\  37--45, 2012.

\bibitem[Hammerich(2007)]{hammerich2007sampling}
Hammerich, E.
\newblock Sampling in shift-invariant spaces with gaussian generator.
\newblock \emph{Sampling Theory in Signal and Image Processing}, 6\penalty0
  (1):\penalty0 71--86, 2007.

\bibitem[Holmes et~al.(2012)Holmes, Lumley, Berkooz, and
  Rowley]{holmes2012turbulence}
Holmes, P., Lumley, J.~L., Berkooz, G., and Rowley, C.~W.
\newblock \emph{Turbulence, coherent structures, dynamical systems and
  symmetry}.
\newblock Cambridge university press, 2012.

\bibitem[Isamiddin et~al.(2021)Isamiddin, Mamasodikova, Khalmatov, Kadirova,
  Mirjalilov, and Primova]{isamiddin2021development}
Isamiddin, S., Mamasodikova, N., Khalmatov, D., Kadirova, N., Mirjalilov, O.,
  and Primova, G.
\newblock Development of neural network forecasting models of dynamic objects
  from observed data.
\newblock 2021.

\bibitem[Izhikevich(2007)]{izhikevich2007dynamical}
Izhikevich, E.~M.
\newblock \emph{Dynamical systems in neuroscience}.
\newblock MIT press, 2007.

\bibitem[Jaeger \& Haas(2004)Jaeger and Haas]{jaeger2004harnessing}
Jaeger, H. and Haas, H.
\newblock Harnessing nonlinearity: Predicting chaotic systems and saving energy
  in wireless communication.
\newblock \emph{science}, 304\penalty0 (5667):\penalty0 78--80, 2004.

\bibitem[Jiang \& Yeh(2003)Jiang and Yeh]{jiang2003bifurcation}
Jiang, I.-G. and Yeh, L.-C.
\newblock Bifurcation for dynamical systems of planet--belt interaction.
\newblock \emph{International Journal of Bifurcation and Chaos}, 13\penalty0
  (03):\penalty0 617--630, 2003.

\bibitem[Kennel et~al.(1992)Kennel, Brown, and
  Abarbanel]{kennel1992determining}
Kennel, M.~B., Brown, R., and Abarbanel, H.~D.
\newblock Determining embedding dimension for phase-space reconstruction using
  a geometrical construction.
\newblock \emph{Physical review A}, 45\penalty0 (6):\penalty0 3403, 1992.

\bibitem[Kim et~al.(1999)Kim, Eykholt, and Salas]{kim1999nonlinear}
Kim, H., Eykholt, R., and Salas, J.
\newblock Nonlinear dynamics, delay times, and embedding windows.
\newblock \emph{Physica D: Nonlinear Phenomena}, 127\penalty0 (1-2):\penalty0
  48--60, 1999.

\bibitem[Kirby(2001)]{kirby2001geometric}
Kirby, M.
\newblock \emph{Geometric data analysis: an empirical approach to
  dimensionality reduction and the study of patterns}, volume~31.
\newblock Wiley New York, 2001.

\bibitem[Knipp(2016)]{knipp2016advances}
Knipp, D.~J.
\newblock Advances in space weather ensemble forecasting, 2016.

\bibitem[Koon et~al.(2000)Koon, Lo, Marsden, and Ross]{koon2000dynamical}
Koon, W.~S., Lo, M.~W., Marsden, J.~E., and Ross, S.~D.
\newblock Dynamical systems, the three-body problem and space mission design.
\newblock In \emph{Equadiff 99: (In 2 Volumes)}, pp.\  1167--1181. World
  Scientific, 2000.

\bibitem[Ku \& Lee(1995)Ku and Lee]{ku1995diagonal}
Ku, C.-C. and Lee, K.~Y.
\newblock Diagonal recurrent neural networks for dynamic systems control.
\newblock \emph{IEEE transactions on neural networks}, 6\penalty0 (1):\penalty0
  144--156, 1995.

\bibitem[Kutz et~al.(2016)Kutz, Brunton, Brunton, and Proctor]{kutz2016dynamic}
Kutz, J.~N., Brunton, S.~L., Brunton, B.~W., and Proctor, J.~L.
\newblock \emph{Dynamic mode decomposition: data-driven modeling of complex
  systems}.
\newblock SIAM, 2016.

\bibitem[LeCun et~al.(2015)LeCun, Bengio, and Hinton]{lecun2015deep}
LeCun, Y., Bengio, Y., and Hinton, G.
\newblock Deep learning.
\newblock \emph{nature}, 521\penalty0 (7553):\penalty0 436--444, 2015.

\bibitem[Li et~al.(2022)Li, Li, Sitzmann, Agrawal, and Torralba]{li20223d}
Li, Y., Li, S., Sitzmann, V., Agrawal, P., and Torralba, A.
\newblock 3d neural scene representations for visuomotor control.
\newblock In \emph{Conference on Robot Learning}, pp.\  112--123. PMLR, 2022.

\bibitem[Lu et~al.(2021)Lu, Jin, Pang, Zhang, and Karniadakis]{lu2021learning}
Lu, L., Jin, P., Pang, G., Zhang, Z., and Karniadakis, G.~E.
\newblock Learning nonlinear operators via deeponet based on the universal
  approximation theorem of operators.
\newblock \emph{Nature machine intelligence}, 3\penalty0 (3):\penalty0
  218--229, 2021.

\bibitem[Lumley(1967)]{lumley1967structure}
Lumley, J.~L.
\newblock The structure of inhomogeneous turbulent flows.
\newblock \emph{Atmospheric turbulence and radio wave propagation}, pp.\
  166--178, 1967.

\bibitem[Majda et~al.(2009)Majda, Franzke, and Crommelin]{majda2009normal}
Majda, A.~J., Franzke, C., and Crommelin, D.
\newblock Normal forms for reduced stochastic climate models.
\newblock \emph{Proceedings of the National Academy of Sciences}, 106\penalty0
  (10):\penalty0 3649--3653, 2009.

\bibitem[Mezi{\'c}(2013)]{mezic2013analysis}
Mezi{\'c}, I.
\newblock Analysis of fluid flows via spectral properties of the koopman
  operator.
\newblock \emph{Annual Review of Fluid Mechanics}, 45:\penalty0 357--378, 2013.

\bibitem[Mildenhall et~al.(2021)Mildenhall, Srinivasan, Tancik, Barron,
  Ramamoorthi, and Ng]{mildenhall2021nerf}
Mildenhall, B., Srinivasan, P.~P., Tancik, M., Barron, J.~T., Ramamoorthi, R.,
  and Ng, R.
\newblock Nerf: Representing scenes as neural radiance fields for view
  synthesis.
\newblock \emph{Communications of the ACM}, 65\penalty0 (1):\penalty0 99--106,
  2021.

\bibitem[Mukhopadhyay \& Banerjee(2020)Mukhopadhyay and
  Banerjee]{mukhopadhyay2020learning}
Mukhopadhyay, S. and Banerjee, S.
\newblock Learning dynamical systems in noise using convolutional neural
  networks.
\newblock \emph{Chaos: An Interdisciplinary Journal of Nonlinear Science},
  30\penalty0 (10):\penalty0 103125, 2020.

\bibitem[Naik \& Cochran(2012)Naik and Cochran]{naik2012nonlinear}
Naik, M. and Cochran, D.
\newblock Nonlinear system identification using compressed sensing.
\newblock In \emph{2012 Conference Record of the Forty Sixth Asilomar
  Conference on Signals, Systems and Computers (ASILOMAR)}, pp.\  426--430.
  IEEE, 2012.

\bibitem[Nguyen \& Mondelli(2020)Nguyen and Mondelli]{nguyen2020global}
Nguyen, Q.~N. and Mondelli, M.
\newblock Global convergence of deep networks with one wide layer followed by
  pyramidal topology.
\newblock \emph{Advances in Neural Information Processing Systems},
  33:\penalty0 11961--11972, 2020.

\bibitem[Park et~al.(2022)Park, Gajamannage, Jayathilake, and
  Bollt]{park2022recurrent}
Park, Y., Gajamannage, K., Jayathilake, D.~I., and Bollt, E.~M.
\newblock Recurrent neural networks for dynamical systems: Applications to
  ordinary differential equations, collective motion, and hydrological
  modeling.
\newblock \emph{arXiv preprint arXiv:2202.07022}, 2022.

\bibitem[Peherstorfer \& Willcox(2015)Peherstorfer and
  Willcox]{peherstorfer2015online}
Peherstorfer, B. and Willcox, K.
\newblock Online adaptive model reduction for nonlinear systems via low-rank
  updates.
\newblock \emph{SIAM Journal on Scientific Computing}, 37\penalty0
  (4):\penalty0 A2123--A2150, 2015.

\bibitem[Qin et~al.(2019)Qin, Wu, and Xiu]{qin2019data}
Qin, T., Wu, K., and Xiu, D.
\newblock Data driven governing equations approximation using deep neural
  networks.
\newblock \emph{Journal of Computational Physics}, 395:\penalty0 620--635,
  2019.

\bibitem[Rahaman et~al.(2019)Rahaman, Baratin, Arpit, Draxler, Lin, Hamprecht,
  Bengio, and Courville]{rahaman2019spectral}
Rahaman, N., Baratin, A., Arpit, D., Draxler, F., Lin, M., Hamprecht, F.,
  Bengio, Y., and Courville, A.
\newblock On the spectral bias of neural networks.
\newblock In \emph{International Conference on Machine Learning}, pp.\
  5301--5310. PMLR, 2019.

\bibitem[Raissi et~al.(2019)Raissi, Perdikaris, and
  Karniadakis]{raissi2019physics}
Raissi, M., Perdikaris, P., and Karniadakis, G.~E.
\newblock Physics-informed neural networks: A deep learning framework for
  solving forward and inverse problems involving nonlinear partial differential
  equations.
\newblock \emph{Journal of Computational physics}, 378:\penalty0 686--707,
  2019.

\bibitem[Ramasinghe \& Lucey(2022)Ramasinghe and Lucey]{ramasinghe2022beyond}
Ramasinghe, S. and Lucey, S.
\newblock Beyond periodicity: towards a unifying framework for activations in
  coordinate-mlps.
\newblock In \emph{Computer Vision--ECCV 2022: 17th European Conference, Tel
  Aviv, Israel, October 23--27, 2022, Proceedings, Part XXXIII}, pp.\
  142--158. Springer, 2022.

\bibitem[Reinbold et~al.(2021)Reinbold, Kageorge, Schatz, and
  Grigoriev]{reinbold2021robust}
Reinbold, P.~A., Kageorge, L.~M., Schatz, M.~F., and Grigoriev, R.~O.
\newblock Robust learning from noisy, incomplete, high-dimensional experimental
  data via physically constrained symbolic regression.
\newblock \emph{Nature communications}, 12\penalty0 (1):\penalty0 3219, 2021.

\bibitem[Sahyoun \& Djouadi(2013)Sahyoun and Djouadi]{sahyoun2013local}
Sahyoun, S. and Djouadi, S.
\newblock Local proper orthogonal decomposition based on space vectors
  clustering.
\newblock In \emph{3rd International Conference on Systems and Control}, pp.\
  665--670. IEEE, 2013.

\bibitem[Schmid(2010)]{schmid2010dynamic}
Schmid, P.~J.
\newblock Dynamic mode decomposition of numerical and experimental data.
\newblock \emph{Journal of fluid mechanics}, 656:\penalty0 5--28, 2010.

\bibitem[Schmit \& Glauser(2004)Schmit and Glauser]{schmit2004improvements}
Schmit, R. and Glauser, M.
\newblock Improvements in low dimensional tools for flow-structure interaction
  problems: using global pod.
\newblock In \emph{42nd AIAA aerospace sciences meeting and exhibit}, pp.\
  889, 2004.

\bibitem[Singer \& Green(2009)Singer and Green]{singer2009using}
Singer, M.~A. and Green, W.~H.
\newblock Using adaptive proper orthogonal decomposition to solve the
  reaction--diffusion equation.
\newblock \emph{Applied Numerical Mathematics}, 59\penalty0 (2):\penalty0
  272--279, 2009.

\bibitem[Sirovich(1987)]{sirovich1987turbulence}
Sirovich, L.
\newblock Turbulence and the dynamics of coherent structures. i. coherent
  structures.
\newblock \emph{Quarterly of applied mathematics}, 45\penalty0 (3):\penalty0
  561--571, 1987.

\bibitem[Sitzmann et~al.(2019)Sitzmann, Zollh{\"o}fer, and
  Wetzstein]{sitzmann2019scene}
Sitzmann, V., Zollh{\"o}fer, M., and Wetzstein, G.
\newblock Scene representation networks: Continuous 3d-structure-aware neural
  scene representations.
\newblock \emph{Advances in Neural Information Processing Systems}, 32, 2019.

\bibitem[Sitzmann et~al.(2020)Sitzmann, Martel, Bergman, Lindell, and
  Wetzstein]{sitzmann2020implicit}
Sitzmann, V., Martel, J., Bergman, A., Lindell, D., and Wetzstein, G.
\newblock Implicit neural representations with periodic activation functions.
\newblock \emph{Advances in Neural Information Processing Systems},
  33:\penalty0 7462--7473, 2020.

\bibitem[Skorokhodov et~al.(2021)Skorokhodov, Ignatyev, and
  Elhoseiny]{skorokhodov2021adversarial}
Skorokhodov, I., Ignatyev, S., and Elhoseiny, M.
\newblock Adversarial generation of continuous images.
\newblock In \emph{Proceedings of the IEEE/CVF Conference on Computer Vision
  and Pattern Recognition}, pp.\  10753--10764, 2021.

\bibitem[Small(2005)]{small2005applied}
Small, M.
\newblock \emph{Applied nonlinear time series analysis: applications in
  physics, physiology and finance}, volume~52.
\newblock World Scientific, 2005.

\bibitem[Stein \& Shakarchi(2011)Stein and Shakarchi]{stein2011fourier}
Stein, E.~M. and Shakarchi, R.
\newblock \emph{Fourier analysis: an introduction}, volume~1.
\newblock Princeton University Press, 2011.

\bibitem[Tancik et~al.(2020)Tancik, Srinivasan, Mildenhall, Fridovich-Keil,
  Raghavan, Singhal, Ramamoorthi, Barron, and Ng]{tancik2020fourier}
Tancik, M., Srinivasan, P., Mildenhall, B., Fridovich-Keil, S., Raghavan, N.,
  Singhal, U., Ramamoorthi, R., Barron, J., and Ng, R.
\newblock Fourier features let networks learn high frequency functions in low
  dimensional domains.
\newblock \emph{Advances in Neural Information Processing Systems},
  33:\penalty0 7537--7547, 2020.

\bibitem[Toni \& Stumpf(2010)Toni and Stumpf]{toni2010simulation}
Toni, T. and Stumpf, M.~P.
\newblock Simulation-based model selection for dynamical systems in systems and
  population biology.
\newblock \emph{Bioinformatics}, 26\penalty0 (1):\penalty0 104--110, 2010.

\bibitem[Tran \& Ward(2017)Tran and Ward]{tran2017exact}
Tran, G. and Ward, R.
\newblock Exact recovery of chaotic systems from highly corrupted data.
\newblock \emph{Multiscale Modeling \& Simulation}, 15\penalty0 (3):\penalty0
  1108--1129, 2017.

\bibitem[Trischler \& D’Eleuterio(2016)Trischler and
  D’Eleuterio]{trischler2016synthesis}
Trischler, A.~P. and D’Eleuterio, G.~M.
\newblock Synthesis of recurrent neural networks for dynamical system
  simulation.
\newblock \emph{Neural Networks}, 80:\penalty0 67--78, 2016.

\bibitem[Tu(2013)]{tu2013dynamic}
Tu, J.~H.
\newblock \emph{Dynamic mode decomposition: Theory and applications}.
\newblock PhD thesis, Princeton University, 2013.

\bibitem[Unser(2000)]{unser2000sampling}
Unser, M.
\newblock Sampling-50 years after shannon.
\newblock \emph{Proceedings of the IEEE}, 88\penalty0 (4):\penalty0 569--587,
  2000.

\bibitem[Uribarri \& Mindlin(2022)Uribarri and Mindlin]{uribarri2022dynamical}
Uribarri, G. and Mindlin, G.~B.
\newblock Dynamical time series embeddings in recurrent neural networks.
\newblock \emph{Chaos, Solitons \& Fractals}, 154:\penalty0 111612, 2022.

\bibitem[Voss et~al.(1999)Voss, Kolodner, Abel, and Kurths]{voss1999amplitude}
Voss, H.~U., Kolodner, P., Abel, M., and Kurths, J.
\newblock Amplitude equations from spatiotemporal binary-fluid convection data.
\newblock \emph{Physical review letters}, 83\penalty0 (17):\penalty0 3422,
  1999.

\bibitem[Wang et~al.(2011)Wang, Yang, Lai, Kovanis, and
  Grebogi]{wang2011predicting}
Wang, W.-X., Yang, R., Lai, Y.-C., Kovanis, V., and Grebogi, C.
\newblock Predicting catastrophes in nonlinear dynamical systems by compressive
  sensing.
\newblock \emph{Physical review letters}, 106\penalty0 (15):\penalty0 154101,
  2011.

\bibitem[Ye et~al.(2015)Ye, Beamish, Glaser, Grant, Hsieh, Richards, Schnute,
  and Sugihara]{ye2015equation}
Ye, H., Beamish, R.~J., Glaser, S.~M., Grant, S.~C., Hsieh, C.-h., Richards,
  L.~J., Schnute, J.~T., and Sugihara, G.
\newblock Equation-free mechanistic ecosystem forecasting using empirical
  dynamic modeling.
\newblock \emph{Proceedings of the National Academy of Sciences}, 112\penalty0
  (13):\penalty0 E1569--E1576, 2015.

\bibitem[Yeung et~al.(2019)Yeung, Kundu, and Hodas]{yeung2019learning}
Yeung, E., Kundu, S., and Hodas, N.
\newblock Learning deep neural network representations for koopman operators of
  nonlinear dynamical systems.
\newblock In \emph{2019 American Control Conference (ACC)}, pp.\  4832--4839.
  IEEE, 2019.

\bibitem[Zayed(2018)]{zayed2018advances}
Zayed, A.~I.
\newblock \emph{Advances in Shannon’s sampling theory}.
\newblock Routledge, 2018.

\end{thebibliography}
\bibliographystyle{icml2023}

\newpage
\appendix
\onecolumn

\section{Sampling theory in higher dimensions}
\label{sec:higher dimensions}
The sampling theory --- in its original form --- is only applicable to one dimensional signals. However, it can be extended to higher dimensions in a straightforward manner. Let 
$f : \R^n \rightarrow \R$ be a function in $L^1(\R^n)$, which we think of as a higher mode signal. Let $I(\Omega_1,\ldots ,\Omega_n)$ denote an n-dimensional rectangle about the origin with side lengths $\Omega_1,\ldots ,\Omega_n$. Suppose that the Fourier transform $\widehat{f}$ vanishes identically outside of  
$I(\Omega_1,\ldots ,\Omega_n)$. Then
\begin{align*}
	f(t_1,\ldots ,t_n) = 
	\sum_{m_1=-\infty}^{\infty}\cdots 
 \sum_{m_n-\infty}^{\infty}
	f\bigg{(}\frac{m_1}{2\Omega_1},\ldots,
	\frac{m_n}{2\Omega_n} \bigg{)} \nonumber
	sinc\big{(}2\Omega_1\big{(}t_1 - \frac{n}{2\Omega_1} \big{)}\big{)}\cdots
	sinc\big{(}2\Omega_n\big{(}t_n - \frac{n}{2\Omega_n} \big{)}\big{)}. 
\end{align*}

Thus we see that sampling $f$ on the lattice defined by lengths 
$\bigg{(}\frac{1}{2\Omega_1},\ldots ,\frac{1}{2\Omega_n} \bigg{)}$ and taking
shifted sinc functions of bandwidth $2\Omega_{k}$, for $1 \leq k \leq n$, we can reconstruct the function $f$ as in the one dimensional case. Note that as in the case of the one-dimensional Nyquist-Shanon theorem, in order for perfect reconstruction one needs to sample at larger than twice the dominant frequency present in the signal. Therefore, in practise one would take the maximum of $ \Omega = \max_i\{\Omega_i\}$ and sample at a frequency of 
$2\Omega$.

\textbf{Curse of dimensionality.} While the multidimensional Nyquist-Shanon sampling theorem provides a convenient theoretical framework in which to understand signal processing problems in higher dimensions. It does not come without problems. In practise, the multidimensional sampling theorem is extremely inefficient. 

The main issue with sampling in higher dimensions is that there is an exponential increase in volumes of cubes (or rectangles/balls) associated with adding extra dimensions. To see this, imagine we had a signal 
$f : [0,1] \rightarrow \R$ whose dominant frequency was $50$-Hertz. Let us then suppose we wish to perform a reconstruction by using a sample rate of $100$-Hertz. This means that we would need to sample exactly $10^2 = 100$ points from the unit interval $[0,1]$ each spaced at a distance of $0.01$. Now, imagine that we had a 10 mode signal $g : [0,1]^{10} \rightarrow \R$ on the unit cube whose dominant frequency was also $50$-Hertz. We wish to perform a $100$-Hertz sample rate reconstruction of $g$ as we did for $f$. Now we see a problem, in this instance we would need to sample $(10^2)^10 = 10^20$ points from the 10-dimensional cube. Thus when using a sampling distance of $0.01$ we see that the 10-dimensional cube $[0,1]^10$ is $10^{18}$-times larger than the 1-dimensional cube $[0,1]$.
This exponential increase in the amount of sample points needed to reconstruct a high mode signal is referred to as the curse of dimensionality and is a mathematical consequence of the fact that volumes of many mathematical shapes grow exponentially with dimension. This makes the sampling theory of Nqyquist and Shanon some what unusable in practise for higher mode signals.

There have been other reconstruction techniques, most notable compressed sensing, that have shown far superior performance than classical sampling due to their ability to break the Nyquist limit and allow far fewer sampling points. However, such techniques have the added problem that they are memory intensive for high mode signals. As we show coordinate neural networks offer a convenient middle ground that makes them perfectly suitable for signal reconstruction in higher mode signal settings.

\section{Riesz bases and sampling.}\label{app:riesz}

In this section, we outline the definition of a Riesz basis in detail and then show which activations generate a Riesz basis and can be used as generators for reconstructing signals.
A reference for this section is \cite{unser2000sampling}.

We recall the definition of a Riesz basis. We fix a function $F$ and consider the space
\begin{equation}\label{riesz_basis_defn}
 V(F) = \bigg{\{} s(x) = \sum_{k =-\infty}^{\infty}a(k)F(x-k) : a \in l^2  \bigg{\}},
\end{equation}
where $l^2$ denotes the space of square summable sequences. This means that a function 
$s \in V$ is determined by its coefficients $c(k)$, provided it is continuously defined.

We say that the family of functions $\{F_k = F(x-k) \}_{k \in \Z}$ defines a Riesz basis if the following property holds. There exists two positive constants
$0 < A, B < \infty$ such that
\begin{equation}\label{riesz_basis_condition}
    A\cdot\vert\vert c\vert\vert^2_{l^2} \leq \bigg{\vert}\bigg{\vert}
    \sum_{k\in \Z}c(k)F_k
    \bigg{\vert}\bigg{\vert}^2 \leq 
    B\cdot\vert\vert c\vert\vert_{l_2}^2
\end{equation}
for all sequences $c(k) \in l_2$, where $l_2$ is the space of square summable sequences and $\vert\vert c\vert\vert_{l_2}^2 = \sum_{k}\vert c(k)\vert^2$ is the squared $l_2$ norm.
The lower inequality says that the basis functions must be linearly independent, which implies every signal $s \in V(F)$ is uniquely determined by its coefficients $c(k)$. The upper bound in the inequality implies that the $L^2$-norm of the signal is finite so that $V(F)$ is a subspace of $L^2$.

In order for the model $V(F)$ to be a good model for sampling it should have the ability to approximate any input function arbitrarily close by choosing an appropriate sampling step, analogous to the Nyquist criterion in classical sampling theory. This is equivalent to the \textit{partition of unity} condition
\begin{equation}\label{puc}
    \sum_{k\in \Z}F(x + k) = 1 \text{ for all } x \in \R.
\end{equation}
The reader is referred to \cite{unser2000sampling} for a discussion on how the partition of unity condition leads to a Nyquist type criterion for the space $V(F)$.

\subsection{Gaussian activation}
\label{app:gaussian reconstruct}

Shifts of a Gaussian form a Riesz basis, see \cite{hammerich2007sampling}. On the other hand the Gaussian will not satisfy \eqref{puc}. However,  
by picking the appropriate variance, and using the Poisson summation formula \cite{stein2011fourier} it can be shown that Gaussian's do approximately satisfy \eqref{PUC}, see 
\cite{hammerich2007sampling}. This shows that while shifts of a Gaussian can be used to generate a basis of functions in $L^2$, their ability to approximate a signal arbitrarily close will not be as strong as the sinc function, due to Gaussians satisfying \eqref{puc} only approximately.

\subsection{Sinusoidal activation}\label{app:sinusoid}

 Sinusoidal functions do not form a Riesz basis as they do not define functions in 
 $L^2(\R)$ and furthermore, due to the periodicity of such functions, they do not satisfy 
 \eqref{puc}. However, if we let $L^2[0,1]$ denote the space of square-integrable 
 periodic functions on $[0,1]$. 
Then the following proposition shows that shifted sine functions can be used for reconstruction of periodic signals.

\begin{prop}\label{sine reconstruct}
Let $L^2[0,1]$ denote the square-integrable space of periodic functions on $[0,1]$. Then given any signal $s \in L^2[0,1]$ we have
\begin{equation*}
    s(x) = \sum_{n= 0}^{\infty} a_nsin(nx + \frac{\pi}{2}) + b_nsin(nx)
\end{equation*}
where the equality in the above should be understood as $L^2$-convergence.
\end{prop}
\begin{proof}
    Fourier analysis, see \cite{stein2011fourier}, shows that periodic signals can be reconstructed via a sine and cosine basis. In other words, we have
    \begin{equation*}
        s(x) = \sum_{n= 0}^{\infty} a_ncos(nx) + b_nsin(nx).
    \end{equation*}
    Using the angle formula, $sin(x + \pi/2) = cos(x)$, we see that the first term on the right of the above equality can be written as 
    $a_nsin(nx + \pi/2)$. This completes the proof.
\end{proof}

We thus see that for periodic signals one only needs two shifts of a sinusoidal function to be able to reconstruct a periodic signal. We emphasize the basis will still be infinite dimensional as each shift can have a different frequency.

\subsection{ReLU activation}\label{app:relu}

The following proposition shows that translates of ReLU cannot generate a Riesz basis.
\begin{prop}\label{relu_riesz_details}
The function $ReLU$ cannot generate a Riesz basis.
\end{prop}
\begin{proof}
    We have to show that the set $V(ReLU)$, defined as in \eqref{riesz_basis_defn}, fails to satisfy \eqref{riesz_basis_condition}. We will show that it fails to satisfy the upper bound given in \eqref{riesz_basis_condition}. We consider the sequence $c(k)$ defined as follows
    \begin{equation*}
        c(k) = 
        \begin{cases}
            0 & \text{if}\ k \leq 1 \\
            1/k & \text{if}\ k > 1.
        \end{cases}
    \end{equation*}
   Then it is clear that $c(k) \in l_2$ with 
   $\vert\vert c(k)\vert\vert_{l_2}^2 = 1$. Taking $x = 0$, we see that 
   \begin{equation*}
       ReLU_k(0) = ReLU(-k) =
       \begin{cases}
           0 & \text{if}\ k \leq 0 \\
           k & \text{if}\ k > 0.
       \end{cases}
   \end{equation*}
   In particular, 
   \begin{align*}
       \sum_{k \in \Z}c(k)ReLU_k(0) = \sum_{k > 1}\frac{k}{k} 
       = \sum_{k>1}1 
       = \infty.
    \end{align*}
    Thus we see that the upper bound in \eqref{riesz_basis_condition} cannot hold and the statement of the proposition has been proved.
\end{proof}

Furthermore, the ReLU function fails to satisfy the partition of unity condition 
\eqref{puc} as the following proposition shows.
\begin{prop}
    The ReLU function fails to satisfy the partition of unity condition \eqref{puc}.
\end{prop}
\begin{proof}
    Fix $x \in \R$ and osberve that 
    \begin{align*}
        \sum_{k \in \Z}ReLU(x + k) = \sum_{k \geq -x}ReLU(x + k)
        = \sum_{k \geq -x}(x + k) 
        = \infty.
    \end{align*}
\end{proof}
This result also explains why coordinate networks employing $ReLU$ activations tend to perform poorly with high frequency signal reconstruction. As shown by proposition
\ref{relu_riesz_details} shifted copies of $ReLU$ do not form a Riesz basis and the extreme failure of condition \eqref{puc} means they will only be able to produce low fidelity reconstructions. The key problem with a ReLU activation is that it cannot be modulated primarily because it satisfies the homogeniety property
\begin{equation}\label{relu_homog_cond}
    ReLU(\omega x) = \omega ReLU(x) \text{ for any}\ \omega > 0.
\end{equation}
Thus, shifts and scales of ReLU's will give poor reconstruction for high frequency signals. We also note that as a ReLU function is not periodic, and cannot be modulated as shown in 
\eqref{relu_homog_cond}, it will not be a good reconstructor for periodic signals.

\section{Lipschitz constant of coordinate networks}

\begin{theorem}\label{layer_lipshitz_sinc_proof}
Let $f_L$ denote a neural network emplying a sinc non-linearity, 
$sinc(\omega x) = \frac{sin(\omega x)}{\omega x}$. Then the Lipshitz constant of $f_k$, for $1 \leq k \leq L$ increases as 
$\omega$ increases. In other words, increasing $\omega$ increases the kth-layers Lipshitz constant.
\end{theorem}

\begin{proof}
    Since our activation function is smooth, the Lipshitz constant can be computed via the operator norm of the Jacobian
\begin{equation}\label{jac1}
	\vert\vert f_k\vert\vert_{Lip} = \sup_{x \in \R^d}\vert\vert J(f_k)(x)\vert\vert_{op}.
\end{equation}
The Jacobian $J(f_k)$ is given by the formula, see \cite{nguyen2020global} for derivation,
\begin{equation}
	J(f_k)(x) = \prod_{l=0}^{k-1}D_{k-l}(x)W_{k-l},
\end{equation}
where $D_{k-l}$ denotes the diagonal $n_{k-l}\times n_{k-l}$ matrix  with
entries $\phi'(g_{k-l,j}(x))$ for $1 \leq j \leq n_{k-l}$, and $g_{k-l, j}(x)$ is the pre-activation neuron. Note that $\phi(x) = sinc(\omega x)$ is our activation function, and $\phi'$ denotes the derivative.

We observe that in the case that $\phi(x) = sinc(\omega x) = \frac{sin(\omega x)}{\omega x}$, 
we have that
$\phi'(x) = \omega\frac{cos(\omega x)}{\omega x} - 
\omega\frac{sin(\omega x)}{(\omega x)^2}$. Thus the matrix 
$D_{k-l}$ can be expressed as $\omega\widetilde{D}_{k-l}$, where
$\widetilde{D}_{k-l}$ denotes the diagonal matrix with diagonal entries given
by $\widetilde{\phi}(g_{k-l,j}(x))$, where 
$\widetilde{\phi}(x) = \frac{cos(\omega x)}{\omega x} - 
\frac{sin(\omega x)}{(\omega x)^2}$. In particular, this implies the Jacobian can be written as 
\begin{equation*}
	J(f_k)(x) = 
	\omega^{k}\prod_{l=0}^{k-1}\widetilde{D}_{k-l}(x)W_{k-l}.
\end{equation*}
Therefore, it follows that by changing $\omega$, the operator norm of the Jacobian will scale by a factor of $\omega^k$ and hence the Lipshitz constant will scale by exactly that factor. In particular, by increasing $\omega$ the Lipshitz constant increases by a factor of $\omega^k$.
\end{proof}

\begin{theorem}\label{layer_lipshitz_gauss}
Let $f_L$ denote a neural network with Gaussian non-linearity, 
$e^{-x^2/\omega^2}$. Then the Lipshitz constant of $f_l$ increases as 
$\omega$ decreases. In other words, decreasing $\omega$ increases the lth-layers 
Lipshitz constant.
\end{theorem}

\begin{proof}
Since our activation function is smooth, the Lipshitz constant can be computed via the operator norm of the Jacobian
\begin{equation}
	\vert\vert f_k\vert\vert_{Lip} = \sup_{x \in \R^d}\vert\vert J(f_k)(x)\vert\vert_{op}.
\end{equation}
The Jacobian $J(f_k)$ is given by the formula
\begin{equation}
	J(f_k)(x) = \prod_{l=0}^{k-1}D_{k-l}(x)W_{k-l},
\end{equation}
where $D_{k-l}$ denotes the diagonal $n_{k-l}\times n_{k-l}$ matrix  with
entries $\phi'(g_{k-l,j}(x))$ for $1 \leq j \leq n_{k-l}$, and $g_{k-l, j}(x)$ is the pre-activation neuron.

We observe that in the case that $\phi(x) = e^{-x^2/\omega^2}$, we have that
$\phi'(x) = \frac{-2x}{\omega^2}e^{-x^2/\omega^2}$. Thus the matrix 
$D_{k-l}$ can be expressed as $\frac{1}{\omega^2}\widetilde{D}_{k-l}$, where
$\widetilde{D}_{k-l}$ denotes the diagonal matrix with diagonal entries given
by $\widetilde{\phi}(g_{k-l,j}(x))$, where 
$\widetilde{\phi}(x) = -2xe^{-x^2/\omega^2}$. In particular, this implies the Jacobian can be written as 
\begin{equation*}
	J(f_k)(x) = 
	\frac{1}{\omega^{2k}}\prod_{l=0}^{k-1}\widetilde{D}_{k-l}(x)W_{k-l}.
\end{equation*}
Therefore, it follows that by changing $\omega$, the operator norm of the Jacobian will scale by a factor of $\frac{1}{\omega^{2k}}$ and hence the Lipshitz constant will scale by exactly that factor. In particular, by increasing $\omega$ the Lipshitz constant increases by a factor of $\frac{1}{\omega^{2k}}$.
\end{proof}

We now show how modulating the bandwidth always the stable rank of the feature maps $F_l$ to increase. We will carry out the following analysis for a Sinc activated coordinate network. Similar results also hold for a Gaussian or Sine/Cos activated network.


\begin{theorem}
\label{eq:stable rank}
Let $f_L$ denote a coordinate neural network with activation 
$sinc(\omega x) = \frac{sin(\omega x)}{\omega x}$. Furthermore, fix a training data set $X = \{x_i\}$ sampled from a fixed training distribution $\mathcal{P}$.
Then increasing $\omega$ leads to on average an increase in the stable rank of the feature map $F_l$.
\end{theorem}
\begin{proof}
The rows of the feature map $F_l$ are given by the outputs of the the lth-layer, 
$f_l$. By proposition \ref{layer_lipshitz_sinc} 
we have that the quantity
\begin{equation}
	\max_{x, y \in X, x\neq y}
	\frac{\vert\vert f_l(x) - f_l(y)\vert\vert}{\vert\vert x-y\vert\vert}
\end{equation} 
increases when $\omega$ increases. Expanding the norm on the numerator we find
\begin{align}\label{parallelogram_expansion}
\frac{\vert\vert f_l(x) - f_l(y)\vert\vert^2}{\vert\vert x-y\vert\vert^2} 
=
\frac{\vert\vert f_l(x)\vert\vert^2 + \vert\vert f_l(y)\vert\vert^2 - 
2\langle f_l(x), f_l(y)\rangle}{\vert\vert x-y\vert\vert^2}. 
\end{align}

We observe that since $\vert sinc(\omega x)\vert \leq 1$, the sum 
$\vert\vert f_l(x)\vert\vert^2 + \vert\vert f_l(y)\vert\vert^2$ is bounded above
by $2n_l$. 
We then observe that as $\omega$ increases the maximum of the right side of 
\eqref{parallelogram_expansion} increases on average, which can only be possible if the inner product is decreasing on average. Thus we see that as  
$\omega$ increases, the rows of the feature matrix $F_l$ are tending to be orthogonal. In other words, the rows of the feature matrix are becoming more spread out. This implies that the energy of the singular values of $F_l$ are more spread out, hence the quantity $\sum_{i}\sigma_i^2$ increases as we increase
$\omega$. By definition of the stable rank, it follows that the stable rank increases.

\end{proof}



\begin{prop}
Let $f_L$ denote a coordinate neural network with activation 
$e^{-x^2/\omega^2}$. Furthermore, fix a training data set $X = \{x_i\}$ sampled from a fixed training distribution $\mathcal{P}$.
Then decreasing $\omega$ leads to on average an increase in the stable rank of the feature map $F_l$.
\end{prop}
\begin{proof}
The rows of the feature map $F_l$ are given by the outputs of the the lth-layer, 
$f_l$. By proposition \ref{layer_lipshitz_gauss} 
we have that the quantity
\begin{equation}
	\max_{x, y \in X, x\neq y}
	\frac{\vert\vert f_l(x) - f_l(y)\vert\vert}{\vert\vert x-y\vert\vert}
\end{equation} 
increases when $\omega$ decrease. Expanding the norm on the numerator we find
\begin{equation}\label{parallelogram_expansion_app}
\frac{\vert\vert f_l(x) - f_l(y)\vert\vert^2}{\vert\vert x-y\vert\vert^2} =
\frac{\vert\vert f_l(x)\vert\vert^2 + \vert\vert f_l(y)\vert\vert^2 - 
2\langle f_l(x), f_l(y)\rangle}{\vert\vert x-y\vert\vert^2}.
\end{equation}
We observe that since $\vert e^{-x^2/\omega^2}\vert \leq 1$, the sum 
$\vert\vert f_l(x)\vert\vert^2 + \vert\vert f_l(y)\vert\vert^2$ is bounded above
by $2n_l$. 
We then observe that as $\omega$ decreases the maximum of the right side of 
\eqref{parallelogram_expansion_app} increases on average, which can only be possible if the inner product is decreasing on average. Thus we see that as  
$\omega$ decreases, the rows of the feature matrix $F_l$ are tending to be orthogonal. In other words, the rows of the feature matrix are becoming more spread out. This implies that the energy of the singular values of $F_l$ are more spread out, hence the quantity $\sum_{i}\sigma_i^2$ increases as we decrease 
$\omega$. By definition of the stable rank, it follows that the stable rank increases.

\end{proof}

\begin{figure}
    \centering
    \includegraphics[width=1.\columnwidth]{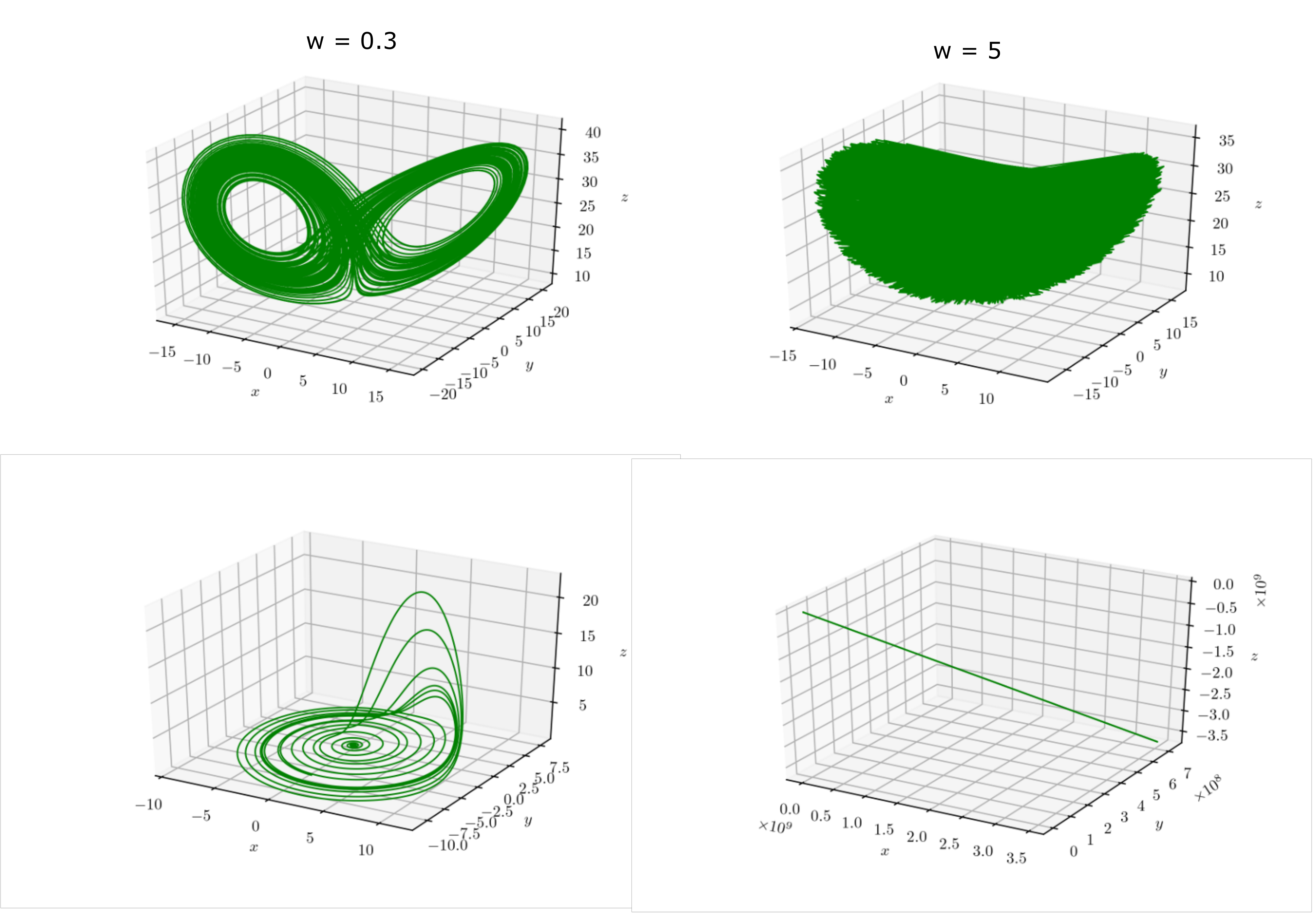}
    \caption{The top row and the bottom row depicts the SINDy reconstructions obtained for the Lorenz system and the Rossler system, respectively, using coordinate networks. As $\omega$ is increased in the sinc function, the coordinate network allows more higher frequencies to be captured, resulting in noisy reconstructions. }
    \label{fig:omega}
\end{figure}

\section{Comparison of activations}

We compare the performance of activation functions in improving SINDy. Fig.\ref{fig:sindyq} shows the results. As depicted, the sinc activation achieves the best results. 

\begin{figure}
    \centering
    \includegraphics[width=1.\columnwidth]{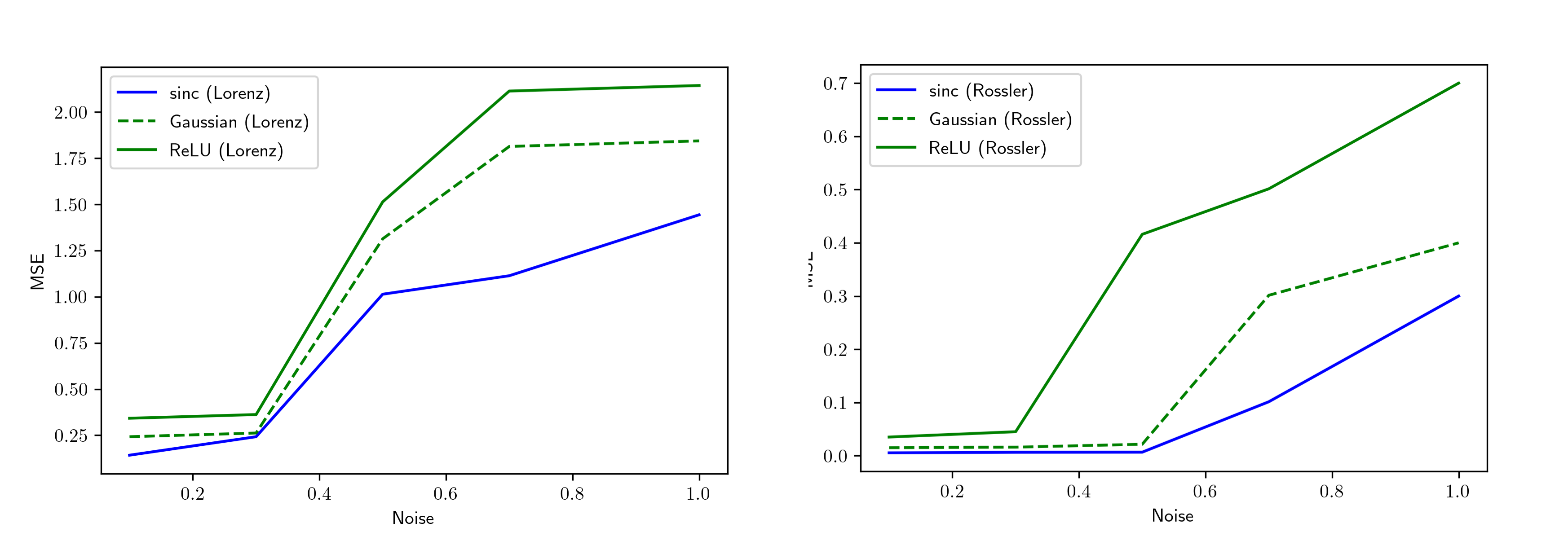}
    \caption{Performance comparison of activation functions in improving SINDy. Sinc activation yields the best results.}
    \label{fig:sindyq}
\end{figure}

\section{Dynamical equations}\label{ds_eqns}

\textbf{Lorentz System:} For the Lorenz system we take the parameters, 
$\sigma = 10$, $\rho = 28$ and $\beta = \frac{8}{3}$. The equations defining the 
system are:
\begin{align}
\frac{dx}{dt} &= \sigma(-x + y) \label{lorenz_sys_eq1} \\
\frac{dy}{dt} &= -xz + \rho x - y \label{lorenz_sys_eq2} \\
\frac{dz}{dt} &= -xy - \beta z \label{lorenz_sys_eq3} 
\end{align}

\textbf{Van der Pol Oscillator:} For the Van der Pol oscillator we take the parameter, 
$\mu = 1$. The equations defining the system are:
\begin{align}
\frac{dx}{dt} &= \mu(x - \frac{1}{3}x^3 - y) \label{vander_sys_eq1} \\
\frac{dy}{dt} &= \frac{1}{\mu}x \label{vander_sys_eq2} 
\end{align}

\textbf{Chen System:} For the Chen system we take the parameters, 
$\alpha = 5$, $\beta = -10$ and $\delta = -0.38$. The equations defining the 
system are:
\begin{align}
\frac{dx}{dt} &= \alpha x - yz \label{chen_sys_eq1} \\
\frac{dy}{dt} &= \beta y + xz \label{chen_sys_eq2} \\
\frac{dz}{dt} &= \delta z + \frac{xy}{3} \label{chen_sys_eq3} 
\end{align}

\textbf{R\"{o}ssler System:} For the R\"{o}ssler system we take the parameters, 
$a = 0.2$, $b = 0.2$ and $c = 5.7$. The equations defining the 
system are:
\begin{align}
\frac{dx}{dt} &= -(y + z) \label{ross_sys_eq1} \\
\frac{dy}{dt} &= x + ay \label{ross_sys_eq2} \\
\frac{dz}{dt} &= b + z(x-c) \label{ross_sys_eq3} 
\end{align}

\newpage

\textbf{Generalized Rank 14 Lorentz System:} For the following system we take parameters
$a = \frac{1}{\sqrt{2}}$, $R = 6.75r$ and $r = 45.92$. The equations defining the system are:
\begin{align}
    \frac{d\psi_{11}}{dt} &= -a\left(\frac{7}{3}\psi_{13}\psi_{22} + \frac{17}{6}\psi_{13}\psi_{24} + \frac{1}{3}\psi_{31}\psi_{22} + \frac{9}{2}\psi_{33}\psi_{24}\right) - 
    \sigma\frac{3}{2}\psi_{11} + \sigma a \frac{2}{3}\theta_{11} \label{rank_14_eq1} \\
    \frac{d\psi_{13}}{dt} &= a\left(-\frac{9}{19}\psi_{11}\psi_{22} + 
    \frac{33}{38}\psi_{11}\psi_{24} + \frac{2}{19}\psi_{31}\psi_{22} - \frac{125}{38}
    \psi_{31}\psi_{24}
    \right) - \sigma\frac{19}{2}\psi_{13} + \sigma a\frac{2}{19}\theta_{13} \label{rank_14_eq2} \\
    \frac{d\psi_{22}}{dt} &= a\left(\frac{4}{3}\psi_{11}\psi_{13} - \frac{2}{3}\psi_{11}\psi_{31} - \frac{4}{3}\psi_{13}\psi_{31} \right) - 
    6\sigma\psi_{22} + \frac{1}{3}\sigma a\theta_{22} \label{rank_14_eq3} \\
    \frac{d\psi_{31}}{dt} &= a\left(\frac{9}{11}\psi_{11}\psi_{22} + \frac{14}{11}
    \psi_{13}\psi_{22} + \frac{85}{22}\psi_{13}\psi_{24}\right) - \frac{11}{2}\sigma
    \psi_{31} + \frac{6}{11}\sigma a\theta_{31} \label{rank_14_eq4} \\
    \frac{d\psi_{33}}{dt} &= a\left(\frac{11}{6}\psi_{11}\psi_{24} \right) - 
    \frac{27}{2}\sigma\psi_{33} + \frac{2}{9}\sigma a\theta_{33} \label{rank_14_eq5} \\
    \frac{d\psi_{24}}{dt} &= a\left(-\frac{2}{9}\psi_{11}\psi_{13} - \psi_{11}\psi_{33} + \frac{5}{9}\psi_{13}\psi_{31} \right) - 18\sigma\psi_{24} + \frac{1}{9}\sigma a\theta_{24} \label{rank_14_eq6} \\
    \frac{d\theta_{11}}{dt} &= a\bigg{(}\psi_{11}\theta_{02} + \psi_{13}\theta_{22} - 
    \frac{1}{2}\psi_{13}\theta_{24} - \psi_{13}\theta_{02} + 2\psi_{13}\theta_{04} 
    + \psi_{22}\theta_{13} + \psi_{22}\theta_{31} + \psi_{31}\theta_{22} \label{rank_14_eq7}\\
    &\hspace{2cm} +
    \frac{3}{2}\psi_{33}\theta_{24} - \frac{1}{2}\psi_{24}\theta_{13} + 
    \frac{3}{2}\psi_{24}\theta_{33} 
    \bigg{)}
    + Ra\psi_{11} - \frac{3}{2}\theta_{11} \nonumber \\
    \frac{d\theta_{13}}{dt} &= a\bigg{(} 
    -\psi_{11}\theta_{22} + \frac{1}{2}\psi_{11}\theta_{24} - \psi_{11}\theta_{02} + 
    2\psi_{11}\theta_{04} - \psi_{22}\theta_{11} - 2\psi_{31}\theta_{22} \label{rank_14_eq8}\\ 
    &\hspace{2.5cm}+ 
    \frac{5}{2}\psi_{31}\theta_{24} + \frac{1}{2}\psi_{24}\theta_{11} + 
    \frac{5}{2}\psi_{24}\theta_{31}
    \bigg{)} + Ra\psi_{13} - \frac{19}{2}\theta_{13} \nonumber \\
    \frac{d\theta_{22}}{dt} &= a\bigg{(} 
    \psi_{11}\theta_{13} - \psi_{11}\theta_{31} - \psi_{13}\theta_{11} + 2\psi_{13}\theta_{31} + 4\psi_{22}\theta_{04} - \psi_{33}\theta_{11} + 
    2\psi_{24}\theta_{02}
    \bigg{)} + 2Ra\psi_{22} - 6\theta_{22} \label{rank_14_eq9} \\
     \frac{d\theta_{31}}{dt} &= a\bigg{(} 
    \psi_{11}\theta_{22} - 2\psi_{13}\theta_{22} + \frac{5}{2}\psi_{13}\theta_{24} - 
    \psi_{22}\theta_{11} + 2\psi_{22}\theta_{13} + 4\psi_{31}\theta_{02} - 
    4\psi_{33}\theta_{02} \label{rank_14_eq10} \\ 
    &\hspace{2cm}+ 8\psi_{33}\theta_{04} - \frac{5}{2}\psi_{24}\theta_{13}    
    \bigg{)} + 3Ra\psi_{31} - \frac{11}{2}\theta_{31} \nonumber \\
    \frac{d\theta_{33}}{dt} &= a\bigg{(}\frac{3}{2}\psi_{11}\theta_{24} - 
    4\psi_{31}\theta_{02} + 8\psi_{31}\theta_{04} - \frac{3}{2}\psi_{24}\theta_{11} 
    \bigg{)} + 3Ra\psi_{33} - \frac{27}{2}\theta_{33} \label{rank_14_eq11} \\
     \frac{d\theta_{24}}{dt} &= a\bigg{(}\frac{1}{2}\psi_{11}\theta_{13} - \frac{3}{2}\psi_{11}\theta_{33} + \frac{1}{2}\psi_{13}\theta_{11} - \frac{5}{2}\psi_{13}\theta_{31} - 2\psi_{22}\theta_{02}  \label{rank_14_eq12} \\
     &\hspace{3cm}- \frac{5}{2}\psi_{31}\theta_{13}
    - \frac{3}{2}\psi_{33}\theta_{11}
     \bigg{)} + 2Ra\psi_{24} - 18\theta_{24} \nonumber \\
     \frac{d\theta_{02}}{dt} &= a\bigg{(}-\frac{1}{2}\psi_{11}\theta_{11} + 
     \frac{1}{2}\psi_{11}\theta_{11} + \frac{1}{2}\psi_{11}\theta_{13} + 
     \frac{1}{2}\psi_{13}\theta_{11} + \psi_{22}\theta_{24} \label{rank_14_eq13} \\
     &\hspace{3cm}- 
     \frac{3}{2}\psi_{31}\theta_{31} + \frac{3}{2}\psi_{31}\theta_{33} + \frac{3}{2}\psi_{33}\theta_{31} + \psi_{24}\theta_{24}
     \bigg{)} - 4\theta_{02} \nonumber \\
     \frac{d\theta_{04}}{dt} &= -a\bigg{(}\psi_{11}\theta_{13} + \psi_{13}\theta_{11} + 2\psi_{22}\theta_{22} + 4\psi_{31}\theta_{33} + 4\psi_{33}\theta_{31}\bigg{)} - 
     16\theta_{04}  \label{rank_14_eq14} 
\end{align}

\section{On Taken's embedding theorem}
\label{sec:taken}

Taken's embedding theorem is a delay embedding theorem giving conditions under which the strange attractor of a dynamical system can be reconstructed from a sequence of observations of the phase space of that dynamical system.

The theorem constructs an embedding vector for each point in time
\begin{equation*}
	x(t_i) = [x(t_i), x_(t_i + n\Delta t),\ldots ,x(t_i + (d-1)n\Delta t)]
\end{equation*}

Where $d$ is the embedding dimension and $n$ is a fixed value.
The theorem then states that in order to reconstruct the dynamics in phase space for any $n$ the following condition must be met
\begin{equation*}
 d \geq 2D + l
\end{equation*}
where $D$ is the box counting dimension of the strange attractor of the dynamical system which can be thought of as the theoretical dimension of phase space for which the trajectories of the system do not overlap. 

\textbf{Problems with the theorem:} The theorem does not provide conditions as to what the best $n$ is and in practise when $D$ is not known it does not provide conditions for the embedding dimension $d$. The quantity $n\Delta t$ is the amount of time delay that is being applied. Extremely short time delays cause the values in the embedding vector to almost be the same, and extremely large time delays cause the value to be uncorrelated random variables. The following papers show how one can find the time delay in practise 
\cite{kim1999nonlinear, small2005applied}.
Furthermore, in practise estimating the embedding dimension is often done by a false nearest neighbours algorithm \cite{kennel1992determining}.

Thus in practise time delay embeddings for the reconstruction of dynamics can require the need to carry further experiments to find the best time delay length and embedding dimension.


\begin{figure}
    \centering
    \includegraphics[width=1.\columnwidth]{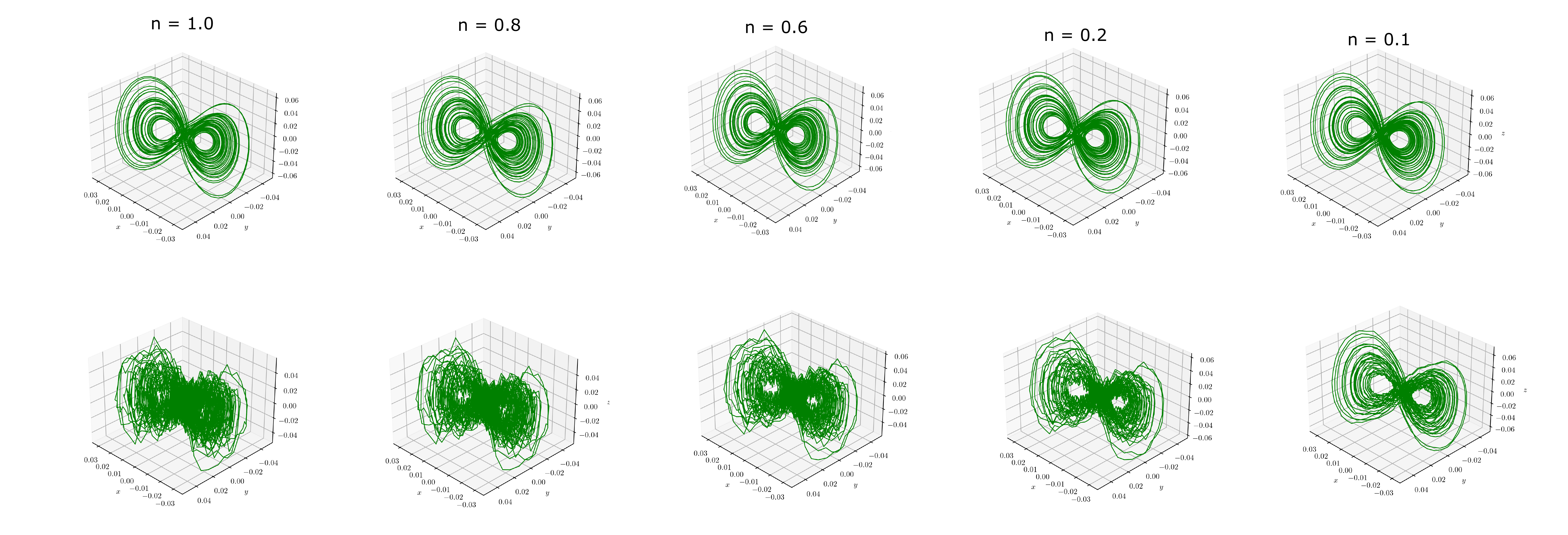}
    \caption{Robust recovery of dynamical systems from partial observations (Lorenz system). \textit{Top row:} coordinate network. \textit{Bottom row:} classical method.}
    \label{fig:diff_lorenz}
\end{figure}

\begin{figure}
    \centering
    \includegraphics[width=1.\columnwidth]{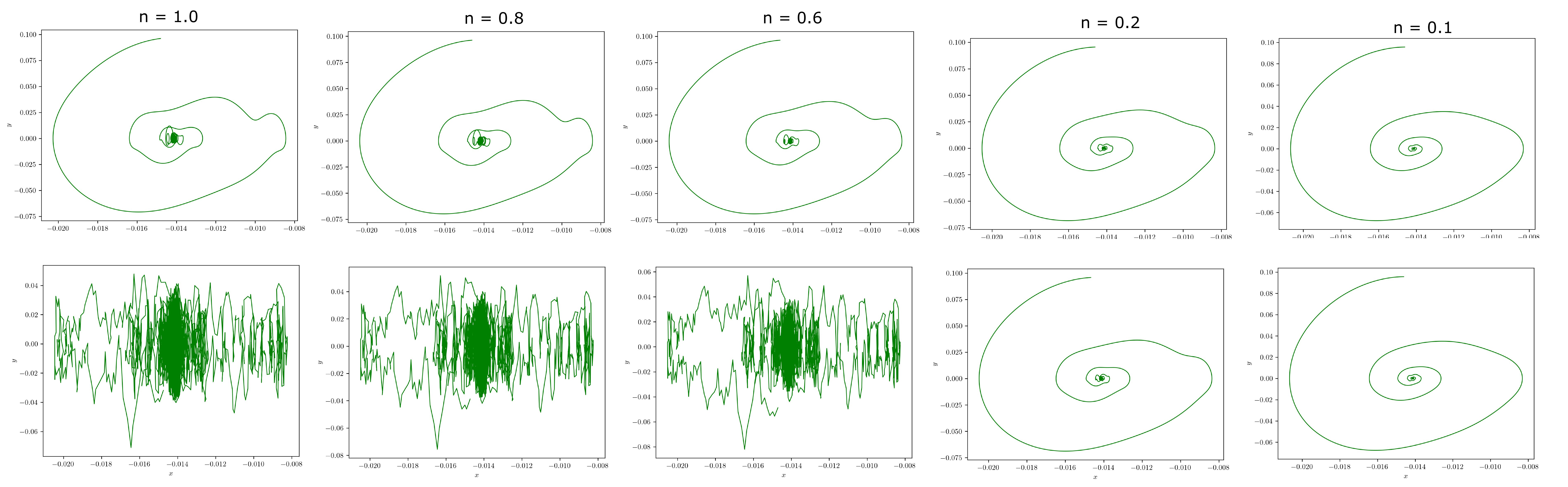}
    \caption{Robust recovery of dynamical systems from partial observations (Duffing system). \textit{Top row:} coordinate network. \textit{Bottom row:} classical method.}
    \label{fig:diff_duffing}
\end{figure}

\end{document}